\documentclass[a4paper,11pt]{article}

\usepackage[margin=1in]{geometry}
\usepackage{amsmath,amssymb,amsthm,mathtools,bm}
\usepackage{booktabs,array,graphicx}
\usepackage{microtype,authblk}
\usepackage{xcolor}
\usepackage{enumitem}
\usepackage{hyperref}
\usepackage[nameinlink,capitalize,noabbrev]{cleveref}
\usepackage{algorithm}
\usepackage{algpseudocode}
\usepackage{pgfplots}
\pgfplotsset{compat=1.18}

\hypersetup{
  colorlinks=true,
  linkcolor=blue!55!black,
  citecolor=blue!55!black,
  urlcolor=blue!55!black,
  pdftitle={BRANCH-MoE: Balance-Aware Tree Routing for Large Embedding Models}
}
\usepackage{natbib}

\newtheorem{theorem}{Theorem}[section]
\newtheorem{lemma}[theorem]{Lemma}
\newtheorem{proposition}[theorem]{Proposition}
\newtheorem{corollary}[theorem]{Corollary}
\newtheorem{assumption}[theorem]{Assumption}
\crefname{assumption}{Assumption}{Assumptions}
\Crefname{assumption}{Assumption}{Assumptions}
\theoremstyle{definition}

\newcommand{\E}{\mathbb{E}}
\newcommand{\Pp}{\mathbb{P}}
\newcommand{\R}{\mathbb{R}}
\newcommand{\one}{\mathbf{1}}
\newcommand{\cJ}{\mathcal{J}}
\newcommand{\cE}{\mathcal{E}}

\newcommand{\pathset}{\operatorname{path}}

\newcolumntype{L}[1]{>{\raggedright\arraybackslash}p{#1}}
\newcommand{\tabsd}[1]{{\scriptsize$\pm$#1}}
\newcommand{\tabmap}[1]{\,{\tiny#1}}

\title{BRANCH-MoE: Balance-Aware Tree Routing for Large Embedding Models}

\author[1]{Gang Fu}
\author[1,2]{Adel Javanmard}
\author[1]{MohammadHossein Bateni}
\author[1]{Vahab Mirrokni}
\affil[1]{Google Research, \texttt{\{thomasfu,bateni,mirrokni\}@google.com}}
\affil[2]{University of Southern California, \texttt{ajavanma@usc.edu}}
\date{}

\newcommand{\branch}{BRANCH-MoE}

\begin{document}
\maketitle

\begin{abstract}
Mixture-of-experts (MoE) layers increase model capacity without a proportional increase in per-example computation. However, conventional flat routers can yield imbalanced expert utilization and treat experts as an unstructured collection, whose indices carry no topological meaning. We introduce {\bf BRANCH-MoE}, a routing architecture that places \(E\) experts at the leaves of a binary decision tree of depth \(\log_2 E\).
At each internal node the branching
probability is centered on the arrival-weighted mean score of the traffic
 reaching that node. This mean is estimated using an exponential moving average, which promotes utilization of both child subtrees without an auxiliary load-balancing loss. We show that this moving-average estimate admits an explicit noise–lag trade-off. We prove that for linear node maps and log-concave arrival distributions, this mechanism prevents routing-mass collapse. We further establish that, under a frozen router, an expert’s execution frequency controls its stochastic-gradient convergence rate, and that confident decisions near the root bound cross-device communication when experts are assigned to devices by tree prefix. We evaluate \branch{} against Switch softmax, DeepSeek-V3 dynamic-bias, Skywork logit-normalized, and deterministic hash routing on Criteo click-through-rate prediction, Forest Covertype, HIGGS, and YearPredictionMSD, using \(E=16\), top-\(4\) routing, and five random seeds. \branch{} matches the predictive performance of the strongest baselines, attaining the lowest mean held-out loss on Forest Covertype and showing no statistically significant difference from the best-performing methods on Criteo, HIGGS, or YearPredictionMSD across random seeds. It also maintains balanced expert load without an auxiliary penalty. Unlike flat routers, \branch{} induces input-conditional co-routing locality: for each input token, the experts selected together tend to lie closer to one another in the routing tree. On all three UCI benchmarks, linear node maps yield a normalized tree distance of \(0.710\)–\(0.718\) between co-selected experts, compared with \(0.817\) for random expert pairs—a reduction of \(12\)–\(13\%\). Every flat-routing baseline remains within \(3.5\%\) of the random-pair distance. Together, these results show that hierarchical routing can preserve task quality and balanced utilization while inducing a topology that supports localized expert co-activation and reduced communication.
\end{abstract}

\section{Introduction}

Mixture-of-experts (MoE; \citealp{jacobs1991adaptive}) layers increase
parameter capacity while activating
only a small subset of expert networks for each token
\citep{shazeer2017outrageously,lepikhin2020gshard,fedus2022switch}. A standard
token-choice layer uses a flat $E$-way router followed by top-$K$ selection.
This design has enabled very large sparse models, but it leaves several
properties to losses, constraints, or systems heuristics rather than building
them into the routing geometry.
\medskip

\noindent\textbf{Limitations of flat routing.}
First, a flat softmax or top-$K$ gate provides no distribution-free positive
lower bound on the mass assigned to every expert: small initial routing
advantages can produce highly unequal update frequencies, and an underused
expert, receiving little data, can improve slowly
\citep{eigen2013learning,shazeer2017outrageously}. Auxiliary
load-balancing losses, capacity constraints, balanced assignment, and
expert-choice routing mitigate this problem
\citep{shazeer2017outrageously,fedus2022switch,lewis2021base,zhou2022expertchoice},
but the protection does not follow from the geometry of a flat gate. Second, a
flat expert label is an unstructured categorical index that does not encode
which experts share a device, node, rack, or other hardware cut, so the router
offers no prefix locality for a distributed implementation to exploit.
Third, preventing early routing collapse requires an auxiliary
load-balancing penalty that competes with the task objective or an external
bias controller
\citep{wang2024auxiliary,deepseekai2024deepseekv3}.
We do not claim that every conventional MoE collapses or
incurs the same communication cost, only that these guarantees do not arise
from flat routing alone.

Our proposal, BRANCH-MoE, changes the routing architecture itself. Its novelty
is the combination of three coupled mechanisms within one routing tree:
(i)~a depth-$D$ binary routing tree for $E=2^D$ experts; (ii)~EMA score
anchors and induced offsets that center every routing decision on the tokens
arriving at its node; and (iii)~a routing-mask cascade whose prefixes can serve
as hierarchical expert or device addresses. The cascade determines the
incoming mask used to update each node's anchor, the anchor determines the
routing offset, and the final mask selects the experts each token is
dispatched to, with the tree path giving each leaf a structured binary address
rather than an arbitrary flat label.
\medskip

\noindent\textbf{Our contributions.} We summarize our contributions as follows:
\begin{enumerate}[leftmargin=*,itemsep=1pt,topsep=2pt]
  \item \textbf{Architecture.} BRANCH-MoE couples a soft mask cascade with one
  scalar EMA score anchor per node, updated with mask weights and a single
  shared decay \(\alpha\) (\cref{sec:architecture}). The anchor induces the
  inference-time bias \(b_j=-\nu_j\), so no auxiliary balancing loss,
  separately learned bias, or \(d\)-dimensional centroid is needed; the full
  soft router evaluates \(E-1\) internal nodes in \(O(\log_2 E)\) sequential
  depth, and inference uses the stored anchors as fixed biases without EMA
  updates.
  \item \textbf{Theory} (\cref{sec:theory}). (1)~The mask-weighted
  constant-decay EMA tracks each node's incoming mean score with an explicit
  noise--lag trade-off. (2)~For linear routing under log-concave arrival
  distributions, anchoring prevents vanishing expected branch mass and yields a
  positive soft leaf-mass bound that degrades quantitatively with anchor error.
  (3)~Under a frozen router and strongly convex conditional expert risks, the
  realized expert-execution frequency controls the convergence rate of expert
  SGD (the soft-mass bound does not itself guarantee top-\(K\) execution).
  (4)~We formalize the communication locality induced by tree routing: a
  token's experts are determined by its path, so co-routed tokens share path
  prefixes and confident upper-level decisions concentrate a token's routing
  mass within a shared-prefix subtree; under tree-aligned expert placement, the
  depth at which two paths diverge bounds whether their dispatches cross a
  hardware boundary, giving explicit bounds on cross-domain communication
  (\cref{sec:communication-locality}).
  \item \textbf{Experiments.} We compare \branch{} against flat softmax (Switch),
  auxiliary-loss-free (DeepSeek-V3), logit-normalized
  (Skywork; \citealp{wei2024skywork}), and
  deterministic hash routing on four tabular benchmarks, namely Criteo
  click-through-rate prediction~\citep{criteo_1tb_click_logs}, Forest Covertype \citep{blackard1999covertype,blackard1998covertype_uci}, HIGGS \citep{baldi2014searching,whiteson2014higgs_uci}, and
  YearPredictionMSD \citep{bertinmahieux2011million,bertinmahieux2011yearpredictionmsd_uci} (see Appendix~\ref{sec:datasets} for description of datasets). As experiments show BRANCH-MoE matches the
  strongest baselines, with the lowest mean held-out loss on Forest Covertype
  and ties within seed noise on Criteo, HIGGS, and MSD. On all three UCI
  benchmarks, it is the \emph{only} router whose co-routing statistics fall
  materially below the analytical random-pair baseline, reducing normalized
  tree distance to $0.710$--$0.718$ against $0.817$ ($12$--$13\%$ lower) and
  the leaf index gap by $18$--$23\%$, while every flat router stays within
  $3.5\%$ of random (\cref{tab:main}).
\end{enumerate}
\vspace{-0.3cm}

\section{Background and Related Work}
\label{app:related}

We first recall the flat
token-choice router that serves as the reference template.
Let $\{F_e(\cdot;\theta_e)\}_{e=1}^E$ be expert networks, where
$\theta_e$ collects the trainable parameters of expert $e$, and let
$W_g\in\R^{d_{in}\times E}$ be a flat gating matrix. A standard token-choice router
forms the score vector $s=x^\top W_g$, whose component
$s_f=x^\top(W_g)_{:,f}$ is the score of expert $f$, and computes
\begin{equation}
  G(x)=\operatorname{softmax}(x^\top W_g).
  \label{eq:flat-softmax}
\end{equation}
After retaining the $K$ largest components, the routed output is
\begin{equation}
  y(x)
  =
  \sum_{e=1}^E
  \operatorname{TopK}(G(x))_e F_e(x).
  \label{eq:flat-output}
\end{equation}

Prior work modifies different aspects of this template to improve sparsity,
load balance, or routing stability. Switch Transformer specializes it to
top-$1$ routing and combines an auxiliary load-balancing loss with fixed expert
capacities \citep{fedus2022switch}. DeepSeek-V3 instead adjusts expert selection
through a dynamic per-expert bias: the bias increases for underloaded experts
and decreases for overloaded experts, while the gate multiplying the expert
output is computed from the original, unbiased affinity score
\citep{deepseekai2024deepseekv3}. Its primary balancing mechanism is therefore auxiliary-loss-free. However, its implementation includes a complementary sequence-wise balancing loss to encourage an even distribution of expert load within each sequence. BRANCH-MoE also uses routing biases, but for a different purpose. Its
bias
\(
  b_j=-\nu_j
\)
is the negated EMA of the arrival-weighted mean score at node $j$, and anchors
that node's decision boundary at the center of its own incoming score
distribution; it is not a load-error feedback variable.

Other methods alter the assignment rule more fundamentally. BASE constructs a
balanced batch assignment through a matching problem
\citep{lewis2021base}, while Sinkhorn-BASE approximates such an assignment using
entropy-regularized optimal transport and Sinkhorn iterations
\citep{clark2022unified}. Expert Choice reverses the usual assignment direction:
each expert selects a fixed-size bucket of high-scoring tokens, controlling
expert workloads while allowing different tokens to be processed by different
numbers of experts \citep{zhou2022expertchoice}. Hash Layers remove the learned
router and assign inputs to experts using a fixed input-dependent hash, thereby
avoiding router parameters, balancing losses, and batchwise matching
\citep{roller2021hash}. These methods respectively impose balance through
batch-coupled assignment, expert-side capacity control, or a fixed routing
rule. BRANCH-MoE instead retains learned, representation-dependent token
routing without solving a globally coupled assignment problem.

BRANCH-MoE is closely related to hierarchical conditional-computation architectures, including hierarchical mixtures of experts \citep{jordan1994hierarchical,shazeer2017outrageously}, soft decision trees \citep{irsoy2012soft,frosst2017soft}, tree-based sparse gating mechanisms \citep{ibrahim2023comet}, and fast feedforward networks \citep{belcak2023fast}. Hierarchical mixtures of experts arrange routing decisions in a tree structure \citep{jordan1994hierarchical}. Similarly, the soft decision tree of Frosst and Hinton uses learned sigmoid splits and path probabilities to distill a neural network into a tree whose leaves represent class distributions \citep{frosst2017soft}. BRANCH-MoE employs the hierarchy differently: its
soft, fully differentiable branch decisions route tokens to trainable neural
experts,
and the same cascaded masks used for routing also determine the tokens used to
update each node's score anchor.

Notably, BRANCH-MoE does not enforce exact per-batch load equality, as do assignment-
constrained methods. Instead, under the conditions stated in our analysis, its
anchored hierarchy provides a population-level nonstarvation
guarantee while preserving adaptive sparse routing. The resulting tree paths
also give experts hierarchical addresses, complementing systems work on token
permutation, distributed communication, and topology-aware routing
\citep{rajbhandari2022deepspeed,hwang2022tutel,he2022fastermoe,chen2022tamoe}.
Notably, prior topology-aware models generally rely on post-hoc routing penalties, dynamic replication, or systems-level scheduling to minimize cross-node traffic, whereas BRANCH-MoE embeds topological locality intrinsically into the geometric structure of the router.

\section{Architecture: BRANCH-MoE}
\label{sec:architecture}

\begin{figure}[t]
\centering
\includegraphics[width=\textwidth]{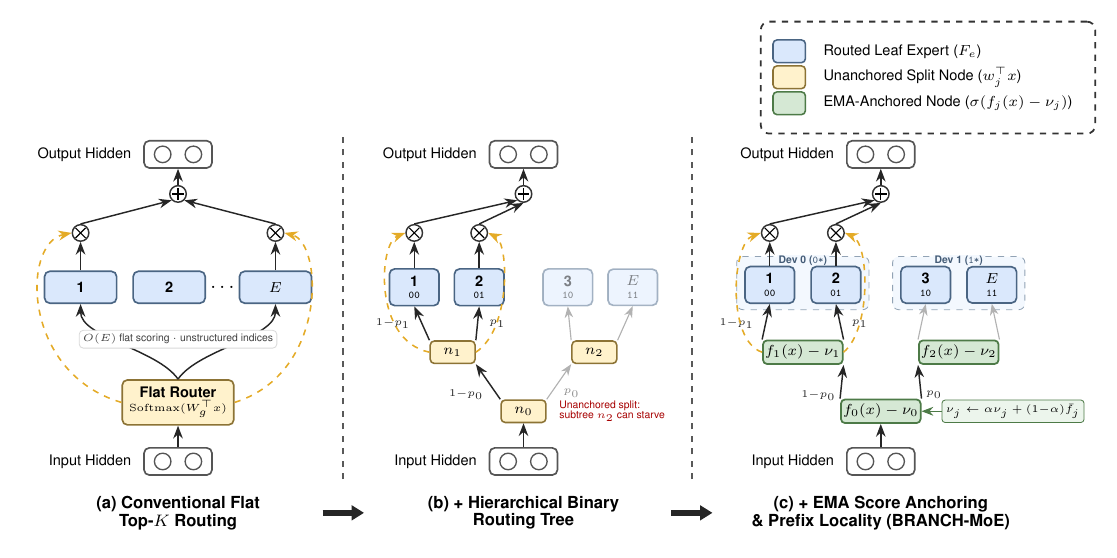}
\caption{Architectural progression from conventional flat routing to \branch{}.
\textbf{(a)} A flat Top-$K$ router scores all $E$ experts independently in
$O(E)$ work and treats expert indices as unstructured labels.
\textbf{(b)} A depth-$D$ binary routing tree with the experts as leaves gives
$O(\log_2 E)$ sequential routing depth and a structured root-to-leaf binary
address per leaf (\texttt{00}, \texttt{01}, \dots), but unanchored internal
splits can starve entire subtrees; the full soft distribution over the leaves
still requires evaluating all $E-1$ internal nodes.
\textbf{(c)} \branch{} centers each internal split $j$ on its arrival-weighted
EMA score anchor ($p_j = \sigma(f_j(x) - \nu_j)$), supporting non-starvation
under the conditions of \cref{sec:score-anticollapse} without an auxiliary
load-balancing loss and inducing prefix-aligned co-routing locality within
hardware devices.}
\label{fig:branch_overview}
\end{figure}

\Cref{fig:branch_overview} illustrates the architecture, a self-balancing
routing tree whose leaves are the experts; \cref{tab:notation}
 summarizes the notation. Let $\mathcal T$ be a complete
binary tree of depth $D$ with internal-node set $\cJ$ and leaf set $\cE$; an
internal node $(d,j)$ has depth $d\in\{0,\ldots,D-1\}$ and index
$j\in\{0,\ldots,2^d-1\}$ within that depth, and for any nonroot node or leaf
$v$, $\pathset(v)$ denotes the internal nodes strictly preceding $v$ on its
root-to-$v$ path. To avoid overloading the depth index, $d_{\rm in}$ denotes
the hidden dimension.

\begin{table}[h]
\centering
\caption{Principal notation.}
\label{tab:notation}
\small
\begin{tabular}{@{}L{0.25\linewidth}L{0.68\linewidth}@{}}
\toprule
Symbol & Definition\\
\midrule
$B,d_{in}$ & Number of resident tokens and hidden dimension\\
$D,\ E=2^D$ & Routing-tree depth and number of leaf experts\\
$\mathcal T,\cJ,\cE$ & Binary tree, internal nodes, and leaves\\
$W\in\R^{(2^D-1)\times d_{in}}$ & Matrix of learned node directions\\
$\nu_j,\ b_j=-\nu_j$ & Node score anchor and induced routing bias\\
$\bar x_j$ & Arrival-weighted centroid of the tokens reaching node $j$\\
$P^{(d)}$ & Routing probabilities at tree depth $d$\\
$M^{(d)}\in\R^{B\times2^d}$ & Incoming routing masks at depth $d$\\
$M_{\rm experts}=M^{(D)}$ & Final expert routing mask\\
$X\in\R^{B\times d_{in}}$ & Input token representations\\
$\alpha\in(0,1)$ & Single EMA decay used for anchor tracking\\
$\mathcal S_K$ & Top-$K$ expert set selected for a token\\
$F_e(\cdot;\theta_e),\theta_e$ & Expert $e$ and its trainable parameter vector\\
\bottomrule
\end{tabular}
\end{table}

\subsection{Phase 1: adaptive geometric anchoring}
\label{sec:phase1}
Each internal node $(d,j)$ carries a trainable scalar-valued \emph{node map}
$f_{j,t}^{(d)}:\R^{d_{\rm in}}\to\R$, evaluated at training step $t$, and a
scalar \emph{anchor} $\nu_{j,t}^{(d)}\in\R$. The parameters of $f_{j,t}^{(d)}$
are optimized by ordinary backpropagation through the routing probabilities,
whereas $\nu_{j,t}^{(d)}$ is a non-gradient state variable maintained by the
EMA \cref{eq:logit-anchor-ema}; there is no independently learned routing bias
or additional scalar calibration variable, since the anchor plays that role
and is refreshed whenever the node statistics change.
\medskip

\noindent\textbf{Node maps: linear and nonlinear projections.}
We instantiate $f_{j,t}^{(d)}$ either as a \emph{linear} projection with a
routing direction $W_{j,t}^{(d)}\in\R^{d_{\rm in}}$ and no bias term, or as a
\emph{nonlinear} per-node two-layer MLP of hidden width $h$:
\begin{align}
  f_{j,t}^{(d)}(x)&=\bigl(W_{j,t}^{(d)}\bigr)^\top x,
  \label{eq:node-map-linear}\\
  f_{j,t}^{(d)}(x)
  &=\bigl(v_{j,t}^{(d)}\bigr)^\top
  \varphi\!\left(U_{j,t}^{(d)}x+c_{j,t}^{(d)}\right),
  \qquad
  U_{j,t}^{(d)}\in\R^{h\times d_{\rm in}},\quad
  v_{j,t}^{(d)},c_{j,t}^{(d)}\in\R^{h},
  \label{eq:node-map-mlp}
\end{align}
with $\varphi$ applied elementwise (ReLU in our experiments); the matrix
$W_t\in\R^{(E-1)\times d_{\rm in}}$ stacks the $\bigl(W_{j,t}^{(d)}\bigr)^\top$
as its rows. Every node owns a private MLP whose output layer carries no bias,
since anchoring removes any additive constant; taking $\varphi=\mathrm{id}$ and
$h=1$ recovers \cref{eq:node-map-linear} up to reparameterization, so linear
routing is the degenerate case of nonlinear routing. The added cost is
$O(hd_{\rm in})$ parameters per node with $h$ small ($h\in\{2,4,8\}$ here),
and the router can still be implemented as a fused batched computation.
\medskip

\noindent\textbf{Arrival-weighted anchoring.}
The anchor statistics and routing masks are constructed jointly in a
root-to-leaf pass. At training step $t$, let $m_{j,t}^{(d)}(x)$ denote the
probability that token $x$ reaches node $(d,j)$, and collect the depth-$d$
masks into
\(
  m_t^{(d)}(x)
  :=
  \bigl(
    m_{0,t}^{(d)}(x),\ldots,
    m_{2^d-1,t}^{(d)}(x)
  \bigr)^\top
  \in\mathbb R^{2^d}
\)
(abbreviated $m^{(d)}(x)$ when the training index is clear); every token
reaches the root with unit mass, so \(m_{0,t}^{(0)}(x)=1\). Given the depth-$d$
masks and a batch $X_t=(x_1,\ldots,x_B)$, the normalized \emph{arrival weight}
of token $x_b$ at node $(d,j)$, the batch \emph{anchor statistic} (the
arrival-weighted mean of the node scores), and the persistent single-decay EMA
anchor are
\begin{align}
  \omega_{j,t}^{(d)}(x_b)
  &:=
  \frac{m_{j,t}^{(d)}(x_b)}
  {\sum_{b'=1}^{B}m_{j,t}^{(d)}(x_{b'})},
  \label{eq:node-arrival-weights}\\
  \widehat\nu_{j,t}^{(d)}
  &:=
  \sum_{b=1}^{B}
  \omega_{j,t}^{(d)}(x_b)\,
  f_{j,t}^{(d)}(x_b),
  \label{eq:logit-anchor-batch}\\
  \nu_{j,t+1}^{(d)}
  &=
  \alpha_t\nu_{j,t}^{(d)}
  +(1-\alpha_t)\widehat\nu_{j,t}^{(d)},
  \qquad
  \alpha_t
  =
  \min(\alpha,\tfrac{t-1}{t}),
  \qquad t\ge1,
  \label{eq:logit-anchor-ema}
\end{align}
where $\alpha\in(0,1)$ is the architectural decay rate and the cap $(t-1)/t$
is a warmup correction that removes initialization bias from the first
updates. Every token thus contributes to every node's anchor in proportion to
the routing mass reaching that node, so the statistic at $(d,j)$ is dominated
by the token population selected by its ancestors; at the root, all arrival
weights equal $1/B$ and $\widehat\nu_{0,t}^{(0)}$ is the ordinary batch mean
of the root scores. Following the batch-normalization convention, the active
anchor differs between training and inference; the centered node logit and
its soft right-branch probability are
\begin{align}
  a_{j,t}^{(d)}
  &:=
  \begin{cases}
    \widehat\nu_{j,t}^{(d)}, & \text{during training},\\[2pt]
    \nu_{j,t}^{(d)}, & \text{during inference},
  \end{cases}
  \label{eq:active-anchor}\\
  z_{j,t}^{(d)}(x)
  &=
  f_{j,t}^{(d)}(x)-a_{j,t}^{(d)},
  \qquad
  p_{j,t}^{(d)}(x)
  =
  \sigma\!\left(z_{j,t}^{(d)}(x)\right),
  \label{eq:centered-logit}
\end{align}
with $\sigma(\cdot)$ the sigmoid function. The live batch statistic makes
training-time centering responsive to changes in the upstream representations,
while the stored EMA makes inference deterministic and independent of batch
composition.

\medskip

\noindent\textbf{Hierarchical mask propagation.}
The incoming masks for depth $d+1$ follow from the branch probabilities at
depth $d$:
\begin{align}
  m_{2j,t}^{(d+1)}(x)
  &=
  m_{j,t}^{(d)}(x)
  \bigl(1-p_{j,t}^{(d)}(x)\bigr),
  \nonumber\\
  m_{2j+1,t}^{(d+1)}(x)
  &=
  m_{j,t}^{(d)}(x)
  p_{j,t}^{(d)}(x),
  \qquad
  j=0,\ldots,2^d-1.
  \label{eq:cascade}
\end{align}
Consequently, $m_{j,t}^{(d)}(x)$ is the product of the branch probabilities
along the unique root-to-node path ending at $(d,j)$. For finite logits, every
sigmoid probability lies strictly between zero and one, so
\(m_{j,t}^{(d)}(x)\in(0,1]\), with equality only at the root: in exact
arithmetic, soft routing sends every token to every node with strictly
positive mass rather than assigning it to a single subtree. The cascade also
preserves total mass, \(\sum_{j=0}^{2^d-1}m_{j,t}^{(d)}(x)=1\) for every depth
$d$, by induction from $m_{0,t}^{(0)}(x)=1$, as the two children of $(d,j)$
receive masses summing to $m_{j,t}^{(d)}(x)$. 

For the batch $X_t=(x_1,\ldots,x_B)$, define
\begin{equation}
  [M_t^{(d)}]_{b,j}
  :=
  m_{j,t}^{(d)}(x_b),
  \qquad
  [P_t^{(d)}]_{b,j}
  :=
  p_{j,t}^{(d)}(x_b),
  \label{eq:batch-mask-definition}
\end{equation}
so that
\[
  M_t^{(d)},P_t^{(d)}
  \in\mathbb R^{B\times 2^d},
  \qquad
  M_t^{(0)}=\one_B.
\]
The dimensionally explicit batch form of \cref{eq:cascade} is
\begin{align}
  [M_t^{(d+1)}]_{:,2j}
  &=
  [M_t^{(d)}]_{:,j}
  \odot
  \bigl(1-[P_t^{(d)}]_{:,j}\bigr),
  \nonumber\\
  [M_t^{(d+1)}]_{:,2j+1}
  &=
  [M_t^{(d)}]_{:,j}
  \odot
  [P_t^{(d)}]_{:,j},
  \qquad
  j=0,\ldots,2^d-1.
  \label{eq:batch-cascade}
\end{align}
Equivalently, after applying the fixed column permutation that groups the left children together and the right children together,
\begin{equation}
  M_t^{(d+1)}
  =
  \left[
    M_t^{(d)}\odot(1-P_t^{(d)})
    \ \middle\|\
    M_t^{(d)}\odot P_t^{(d)}
  \right].
  \label{eq:batch-cascade-concatenated}
\end{equation}

The denominator in \cref{eq:node-arrival-weights} is strictly positive since the probabilities are strictly positive. In reduced precision, however,
a depth-$d$ mask is a product of $d$ sigmoid factors and may underflow to
zero. The implementation therefore clamps the denominator away from zero and
skips the corresponding EMA update if the pre-clamp arrival mass vanishes,
leaving $\nu_{j,t}^{(d)}$ unchanged. This is a numerical safeguard rather than
a modification of the routing rule.
\medskip

\noindent\textbf{What is anchored, and in which space.}
Substituting \cref{eq:logit-anchor-batch} into the training branch of
\cref{eq:centered-logit} gives
\begin{equation}
  \sum_{b=1}^{B}
  \omega_{j,t}^{(d)}(x_b)\,
  z_{j,t}^{(d)}(x_b)
  =0.
  \label{eq:zero-mean-logit}
\end{equation}
This identity holds for any node map $f_{j,t}^{(d)}$, whether linear or
nonlinear: each node centers its scores relative to the population that
actually reaches it, so any routing imbalance must arise from variation of the
scores around this population-specific mean rather than from a shared offset
in the incoming representations. For the linear node map
\[
  f_{j,t}^{(d)}(x)
  =
  \bigl(W_{j,t}^{(d)}\bigr)^\top x,
\]
define the arrival-weighted centroid
\begin{equation}
  \bar x_{j,t}^{(d)}
  :=
  \sum_{b=1}^{B}
  \omega_{j,t}^{(d)}(x_b)x_b.
  \label{eq:arrival-weighted-centroid}
\end{equation}
Linearity gives
\begin{equation}
  \widehat\nu_{j,t}^{(d)}
  =
  \bigl(W_{j,t}^{(d)}\bigr)^\top
  \bar x_{j,t}^{(d)},
  \label{eq:linear-anchor-equivalence}
\end{equation}
and hence the training-time logit becomes
\begin{equation}
  z_{j,t}^{(d)}(x)
  =
  \bigl(W_{j,t}^{(d)}\bigr)^\top
  \bigl(x-\bar x_{j,t}^{(d)}\bigr).
  \label{eq:linear-centered-logit}
\end{equation}
Thus, during training, the separator is the hyperplane with normal
$W_{j,t}^{(d)}$ passing through the centroid of its incoming population.
During inference,
\[
  z_{j,t}^{(d)}(x)
  =
  \bigl(W_{j,t}^{(d)}\bigr)^\top x-\nu_{j,t}^{(d)},
\]
so the stored anchor induces the scalar bias
\begin{equation}
  b_{j,t}^{(d)}=-\nu_{j,t}^{(d)}.
  \label{eq:router-bias}
\end{equation}
Scalar anchoring is therefore exactly equivalent to centroid anchoring for
the live statistic in the linear case, while requiring only $O(1)$ persistent
state per node rather than $O(\dim(x))$.\footnote{For the inference statistic,
the equivalence is affected by the drift of $W_{j,t}^{(d)}$ over the EMA
window: $\nu_{j,t}^{(d)}$ averages past values of
$\bigl(W_{j,\tau}^{(d)}\bigr)^\top\bar x_{j,\tau}^{(d)}$ computed with
potentially different router parameters.}

For a nonlinear node map, the anchored quantity is the scalar node score
rather than a point in token space. At training time, the decision boundary is
the level set
\[
  \left\{
    x:
    f_{j,t}^{(d)}(x)=\widehat\nu_{j,t}^{(d)}
  \right\}.
\]
This boundary may be curved, piecewise linear, or disconnected, depending on
the node map, and need not pass through a centroid in the original token
space. What remains exact is the score-centering identity
\cref{eq:zero-mean-logit}. Accordingly, geometric arguments that require
halfspace sublevel sets apply directly to linear routing, or to routing that
is linear after a shared feature map, but not to an arbitrary nonlinear node
map.

\medskip

\noindent\textbf{Gradient geometry.}
During training, the arrival weights in \cref{eq:node-arrival-weights} are detached when they
are used to form the live anchor. Their forward values determine which tokens
contribute to the node statistic, but gradients from the anchor are not
propagated through the routing masks. The live anchor remains differentiable
through the node scores themselves. Therefore, for the parameters
$\theta_{j,t}^{(d)}$ of node $(d,j)$,
\begin{equation}
  \nabla_{\theta_{j,t}^{(d)}}z_{j,t}^{(d)}(x_b)
  =
  \nabla_{\theta_{j,t}^{(d)}}f_{j,t}^{(d)}(x_b)
  -
  \sum_{b'=1}^{B}
  \omega_{j,t}^{(d)}(x_{b'})
  \nabla_{\theta_{j,t}^{(d)}}f_{j,t}^{(d)}(x_{b'}).
  \label{eq:anchored-grad}
\end{equation}
Thus, each node receives parameter gradients centered by the
arrival-weighted mean gradient of the population it serves. For the linear
map,
\begin{equation}
  \nabla_{W_{j,t}^{(d)}}z_{j,t}^{(d)}(x_b)
  =
  x_b-\bar x_{j,t}^{(d)}.
  \label{eq:linear-anchored-grad}
\end{equation}
For a nonlinear map, the same centering occurs in the corresponding Jacobian
features. Anchoring therefore removes the shared mean component that could
otherwise dominate the router gradient when the incoming representations
contain a large common offset.


\subsection{Phase 2: expert execution}
\label{sec:phase2}

Phase 1 produces the leaf-mask matrix
$M_{\mathrm{experts}}=M^{(D)}\in\mathbb R^{B\times E}$, with
$[M^{(D)}]_{b,e}=m_e^{(D)}(x_b)$. Each leaf probability is the product of the
$D$ branch probabilities along its root-to-leaf path, and every internal node
splits its incoming mass between its two children with probabilities $p$ and
$1-p$, so the router outputs a normalized probability distribution over the
$E$ leaf experts rather than unnormalized expert scores. 

The router returns both this
distribution and the indices of its $K$ largest entries, leaving top-$K$
sparsification to the expert layer: the same routing distribution thus
supports both dense and sparse implementations, and routing statistics and
intended expert loads can be measured before capacity constraints alter the
realized assignments. 

\medskip

\noindent\textbf{Gating.} Let $\mathcal S_K(x)\subseteq\cE$ index the $K$ largest
entries of $m^{(D)}(x)$, with retained mass
$R_K(x):=\sum_{e'\in\mathcal S_K(x)}m_{e'}^{(D)}(x)$. The gate applied to
expert $e$ is given by 
\begin{align}
  \widetilde m_e(x)
  =
  \frac{
    m_e^{(D)}(x)\,
    \one\{e\in\mathcal S_K(x)\}
  }{
    R_K(x)^\eta
  },
  \qquad
  \eta\in\{0,1\},
  \label{eq:gate}
\end{align}
With $\eta=1$, each token's retained gates are renormalized, after top-$K$
masking, to sum to one; with $\eta=0$, they retain their original path
probabilities. Importantly, normalization is applied only after top-$K$ masking;
unselected experts therefore have no effect on the relative weights of the
selected experts.

Regarding the choice of $\eta$ in \cref{eq:gate}, note that
renormalization preserves the relative weights
within $\mathcal S_K(x)$. For $K>1$, this produces a convex combination of the
selected expert outputs and stabilizes the scale of the mixture. For $K=1$,
however, renormalization would force the unique nonzero gate to equal one,
discarding the router's only scalar measure of confidence in the selected
leaf. We therefore use $\eta=\one\{K>1\}$ as the default choice.
\medskip

\noindent \textbf{Execution.} The \branch{} output is the gated combination of the selected experts,
\begin{align}
\widehat y(x)
  =
  \sum_{e\in\mathcal S_K(x)}
  \widetilde m_e(x)\,F_e(x).
  \label{eq:expert-output}
  \end{align}
The sparse implementation evaluates only the $K$ selected experts, grouping
tokens by expert and scattering the gated outputs back to their original
positions. Ignoring dispatch overhead, this requires a
$\Theta(K/E)$ fraction of the expert computation of dense evaluation.


\section{Theoretical Analysis}
\label{sec:theory}

We analyze four components of BRANCH-MoE: tracking of the arrival-weighted
score anchors, prevention of routing-mass collapse, optimization of sparsely
selected experts, and communication locality. 
Proofs are deferred to
\cref{app:proofs}.

\subsection{Arrival-weighted population and anchor tracking}
\label{app:anchor-tracking}

Fix a node $(d,j)$ and temporarily freeze the router and the node score. To
avoid excessive subscripts, write $\widehat\nu_{t}= \widehat\nu_{j,t}^{(d)}$ and $\nu_{t}= \nu_{j,t}^{(d)}$.

\begin{assumption}[Stationary anchor estimates]
\label{ass:stationary-anchor}
There is a deterministic scalar $\bar\nu$ such that
\begin{equation}
  \widehat\nu_t=\bar\nu+\xi_t,
  \qquad
  \E[\xi_t\mid\mathcal F_t]=0,
  \qquad
  \E[\xi_t^2\mid\mathcal F_t]\le s^2.
  \label{eq:anchor-noise}
\end{equation}
Here $\mathcal F_t$ contains the model state and all batches preceding the
$t$-th update, so $\nu_t$ is $\mathcal F_t$-measurable and
$\widehat\nu_t$ is computed from a fresh batch. The initialization $\nu_0$
is deterministic.
\end{assumption}

\begin{theorem}[Stationary behavior of the anchor EMA]
\label{thm:anchor-ema-stationary}
Under \cref{ass:stationary-anchor}, consider the constant-decay phase of
\cref{eq:logit-anchor-ema}, after the warmup has reached
$\alpha_t=\alpha$. Reindexing the first update in this phase as $t=0$, the
anchor satisfies
\begin{align}
  \E[\nu_t-\bar\nu]
  &=
  \alpha^t(\nu_0-\bar\nu),
  \label{eq:anchor-ema-mean}\\
  \E[(\nu_t-\bar\nu)^2]
  &\le
  \alpha^{2t}(\nu_0-\bar\nu)^2
  +\frac{1-\alpha}{1+\alpha}
   (1-\alpha^{2t})s^2.
  \label{eq:anchor-ema-mse}
\end{align}
Thus, the initialization error decays geometrically, while the asymptotic
mean-square fluctuation is at most $(1-\alpha)s^2/(1+\alpha)$.
\end{theorem}

The theorem makes the role of $\alpha$ transparent. A larger $\alpha$ reduces
the sampling-noise floor but makes the EMA slower to respond. Under soft
routing every node receives positive mass in exact arithmetic, but deep-node
arrival weights can be highly concentrated, increasing the variance parameter
$s^2$ and the risk of numerical underflow.

\begin{proposition}[Tracking a drifting anchor]
\label{prop:anchor-ema-drift}
Under the same constant-decay update, suppose
\begin{equation}
  \E[\widehat\nu_t\mid\mathcal F_t]=\bar\nu_t,
  \qquad
  \E[(\widehat\nu_t-\bar\nu_t)^2\mid\mathcal F_t]\le s^2,
  \qquad
  |\bar\nu_{t+1}-\bar\nu_t|\le v,
  \label{eq:drifting-anchor-assumption}
\end{equation}
where $(\bar\nu_t)$ and $\nu_0$ are deterministic. Then
\begin{align}
  \E[(\nu_t-\bar\nu_t)^2]
  \le{}&
  \left(
    \alpha^t|\nu_0-\bar\nu_0|
    +\frac{v(1-\alpha^t)}{1-\alpha}
  \right)^2
  \nonumber\\
  &+
  \frac{1-\alpha}{1+\alpha}
  (1-\alpha^{2t})s^2.
  \label{eq:anchor-drift-bound}
\end{align}
\end{proposition}

The two nontransient terms exhibit the tracking trade-off of the single EMA:
large $\alpha$ suppresses sampling noise, whereas small $\alpha$ reduces lag.


\subsection{Score anchoring and routing-mass anti-collapse}
\label{sec:score-anticollapse}

Fix the soft router, including its active anchors, and a node $(d,j)$, and
suppress the fixed step index. Its expected incoming mass is
$q_j^{(d)}:=\E[m_j^{(d)}(X)]>0$. For every measurable set $\mathcal A$ in the input
space, define the arrival-weighted law by
$\mathcal D_j^{(d)}(\mathcal A):=
\E[m_j^{(d)}(X)\one\{X\in\mathcal A\}]/q_j^{(d)}$. Let
$\nu_j^{(d),\star}:=\E_{\mathcal D_j^{(d)}}[f_j^{(d)}(X)]$ be the population
score anchor and let
$\delta_j^{(d)}:=a_j^{(d)}-\nu_j^{(d),\star}$ denote the active-anchor error.
The relative masses sent to the right and left children are, respectively,
$\pi_{j,+}^{(d)}:=\E_{\mathcal D_j^{(d)}}[
\sigma(f_j^{(d)}(X)-a_j^{(d)})]$ and
$\pi_{j,-}^{(d)}:=\E_{\mathcal D_j^{(d)}}[
\sigma(a_j^{(d)}-f_j^{(d)}(X))]$.

\begin{lemma}[Log-concavity under soft linear routing]
\label{lem:soft-logconcavity}
Suppose the input distribution has a log-concave density on its affine
support and every ancestor of node $(d,j)$ uses an affine score. Then the
arrival-weighted law $\mathcal D_j^{(d)}$ is log-concave.
\end{lemma}

\begin{corollary}[Linear routing under a log-concave arrival law]
\label{cor:linear-soft-anticollapse}
Suppose $\mathcal D_j^{(d)}$ is log-concave on its affine support and
\begin{equation}
  f_j^{(d)}(x)
  =
  \bigl(W_j^{(d)}\bigr)^\top x
  \label{eq:linear-node-score-theory}
\end{equation}
with $W_j^{(d)}$ nonzero on the linear span of that support. At the exact
population anchor,
\begin{equation}
  \min\{\pi_{j,+}^{(d)},\pi_{j,-}^{(d)}\}
  \ge
  \frac{1}{2\mathrm e}.
  \label{eq:linear-soft-anticollapse}
\end{equation}
If the centered projected score has density bounded by $A_j^{(d)}$, then for
an arbitrary active anchor,
\begin{equation}
  \min\{\pi_{j,+}^{(d)},\pi_{j,-}^{(d)}\}
  \ge
  \frac12
  \left(
    \frac1{\mathrm e}
    -A_j^{(d)}|\delta_j^{(d)}|
  \right)_+.
  \label{eq:linear-soft-anticollapse-error}
\end{equation}
\end{corollary}

For a fixed soft router, let
\begin{equation}
  q_e
  :=
  \E[m_e^{(D)}(X)]
  \label{eq:expected-leaf-mass}
\end{equation}
be the expected routing mass of leaf $e$. If the density condition in
\cref{cor:linear-soft-anticollapse} holds at every node on the path to $e$,
define
\begin{equation}
  \rho_j^{(d)}
  :=
  \frac12
  \left(
    \frac1{\mathrm e}
    -A_j^{(d)}|\delta_j^{(d)}|
  \right)_+,
  \label{eq:node-mass-lower-bound}
\end{equation}
then repeated conditioning along the path to $e$ gives
\begin{equation}
  q_e
  \ge
  \underline q_e:=\prod_{(d,j)\in\pathset(e)}\rho_j^{(d)}.
  \label{eq:soft-leaf-mass-lower-bound}
\end{equation}
In the linear log-concave case with exact population anchors,
\begin{equation}
 \underline q_e\ge(2\mathrm e)^{-D}
  \qquad\text{for every leaf }e.
  \label{eq:exact-soft-leaf-bound}
\end{equation}

These bounds do not, by themselves, imply that every leaf is retained by
deterministic top-$K$ truncation; this distinction is addressed next.

\subsection{Expert optimization under sparse top-$K$ routing}
\label{app:expert-optimization}

Freeze the router, the top-$K$ rule, and any capacity policy. Let
\begin{equation}
  I_{e,t}
  :=
  \one\{\text{expert $e$ executes example $t$}\},
  \qquad
  p_e:=\Pp(I_{e,t}=1).
  \label{eq:expert-execution-probability}
\end{equation}
Thus $p_e$ is the actual execution probability after top-$K$ selection, not the expected leaf mass $q_e$ in
\cref{eq:expected-leaf-mass}. Let $\theta_e$ collect the trainable parameters
of expert $F_e$, and let $\mathcal L_e$ denote its conditional risk on the
examples it executes.

\begin{assumption}[Frozen-router expert SGD]
\label{ass:frozen-router-expert-sgd}
The examples are i.i.d. and
$\Pp(I_{e,t}=1\mid\mathcal H_t)=p_e$, where $\mathcal H_t$ is the history
before example $t$, and $\theta_{e,0}$ is deterministic. The conditional risk $\mathcal L_e$ is
$\mu_e$-strongly convex and $L_e$-smooth with minimizer $\theta_e^\star$.
Conditional on $I_{e,t}=1$, the stochastic gradient satisfies
\begin{align}
  \E[g_{e,t}\mid\mathcal H_t,I_{e,t}=1]
  &=
  \nabla\mathcal L_e(\theta_{e,t}),
  \nonumber\\
  \E\left[
    \|g_{e,t}-\nabla\mathcal L_e(\theta_{e,t})\|^2
    \mid\mathcal H_t,I_{e,t}=1
  \right]
  &\le
  \sigma_e^2.
  \label{eq:expert-gradient-assumption-current}
\end{align}
\end{assumption}

The global-example update is
\begin{equation}
  \theta_{e,t+1}
  =
  \theta_{e,t}-\gamma_e I_{e,t}g_{e,t}.
  \label{eq:expert-sgd-update-current}
\end{equation}

\begin{theorem}[Expert progress under intermittent execution]
\label{thm:expert-progress-current}
Under \cref{ass:frozen-router-expert-sgd}, if $p_e>0$ and
$0<\gamma_e\le1/L_e$, then
\begin{align}
  \E[\|\theta_{e,T}-\theta_e^\star\|^2]
  \le{}&
  (1-p_e\gamma_e\mu_e)^T
  \|\theta_{e,0}-\theta_e^\star\|^2
  +\frac{\gamma_e\sigma_e^2}{\mu_e}
  \nonumber\\
  \le{}&
  \exp(-p_e\gamma_e\mu_eT)
  \|\theta_{e,0}-\theta_e^\star\|^2
  +\frac{\gamma_e\sigma_e^2}{\mu_e}.
  \label{eq:expert-convergence-current}
\end{align}
Consequently, if $p_e\ge p_{\min}>0$ for all experts, then, for
$0<\varepsilon<\|\theta_{e,0}-\theta_e^\star\|^2$, the transient is at most
$\varepsilon$ once
\begin{equation}
  T\ge
  \frac{1}{p_{\min}\gamma_e\mu_e}
  \log\!\left(
    \frac{\|\theta_{e,0}-\theta_e^\star\|^2}{\varepsilon}
  \right).
  \label{eq:expert-iteration-complexity}
\end{equation}
\end{theorem}

\paragraph{Implication for BRANCH-MoE.}
\Cref{thm:expert-progress-current} identifies the realized execution frequency
as the optimization bottleneck for sparsely activated experts: a larger minimum
frequency \(p_{\min}\) yields faster worst-expert convergence in global
training steps. BRANCH-MoE is designed to improve this quantity by anchoring
each internal decision to the population reaching that node, thereby
discouraging persistent high-level routing imbalances that would
simultaneously suppress every expert in an entire subtree. Although the
soft-mass bound in \cref{eq:soft-leaf-mass-lower-bound} does not formally
imply a positive top-\(K\) execution probability, it removes an important
mechanism of expert starvation and should make realized execution frequencies
larger and more stable than under unconstrained learned routing. Thus, the
theorem explains how the improved utilization encouraged by BRANCH-MoE can
translate into faster optimization of the least frequently selected experts.
Uniform random routing remains the natural benchmark for marginal balance;
BRANCH-MoE seeks comparable utilization while retaining task-dependent
specialization and tree-aligned communication locality.


\subsection{Communication locality from learned tree routing}
\label{sec:communication-locality}

We finally analyze communication at a hardware boundary whose domains are
aligned with depth-$r$ routing subtrees; a token is transmitted once to every
contacted domain and may then be replicated locally. The argument depends only
on the centered node scores, so it applies to both linear and nonlinear node
maps. Fix a token $x$, follow the more probable child for $r$ levels, and let
\(u_0,u_1,\ldots,u_r\) be the resulting path, where $u_i=(i,j_i)$ and $u_0$ is
the root. Let $\mathcal L(u_r)$ be the leaves below $u_r$, assumed to contain
at least $K$ experts, and let $\mathcal S_K(x)$ be the top-$K$ leaf set. Under
prefix-aligned placement, let $N_r(x)$ denote the number of depth-$r$
communication domains containing at least one selected expert, and suppose a
token representation contains $s_{\rm tok}$ bytes and transmission across this
boundary costs $c_r$ per byte. The learned score margin at $u_i$, the leaf
mass outside the prefix subtree, the smallest retained probability of
$\mathcal S_K(x)$, and the communication cost are
\begin{align}
  \gamma_i(x)
  &:=
  \left|z_{j_i,t}^{(i)}(x)\right|
  =
  \left|
    f_{j_i,t}^{(i)}(x)-a_{j_i,t}^{(i)}
  \right|,
  \label{eq:learned-score-margin}\\
  \varepsilon_r(x)
  &:=
  \sum_{e\notin\mathcal L(u_r)}m_{e,t}^{(D)}(x),
  \label{eq:outside-prefix-mass-current}\\
  \tau_K(x)
  &:=
  \min_{e\in\mathcal S_K(x)}m_{e,t}^{(D)}(x),
  \label{eq:topk-threshold-current}\\
  \operatorname{Comm}_r(x)
  &:=
  s_{\rm tok}c_rN_r(x).
  \label{eq:level-r-communication-current}
\end{align}
At $u_i$, the preferred child receives probability $\sigma(\gamma_i(x))$, while
the other child receives $\sigma(-\gamma_i(x))\le e^{-\gamma_i(x)}$.

\begin{theorem}[Learned-margin communication locality]
\label{thm:learned-margin-communication-current}
For every token $x$,
\begin{equation}
  \varepsilon_r(x)
  =
  1-\prod_{i=0}^{r-1}\sigma(\gamma_i(x))
  \le
  \sum_{i=0}^{r-1}e^{-\gamma_i(x)}.
  \label{eq:prefix-mass-concentration-current}
\end{equation}
Moreover,
\begin{equation}
  \left|
    \mathcal S_K(x)\setminus\mathcal L(u_r)
  \right|
  \le
  \left\lfloor
    \frac{\varepsilon_r(x)}{\tau_K(x)}
  \right\rfloor,
  \label{eq:number-selected-outside-current}
\end{equation}
and
\begin{equation}
  N_r(x)
  \le
  \min\left\{
    K,
    1+
    \left\lfloor
      \frac{\varepsilon_r(x)}{\tau_K(x)}
    \right\rfloor
  \right\}.
  \label{eq:communication-domain-bound-current}
\end{equation}
In particular, if
\begin{equation}
  \sum_{i=0}^{r-1}e^{-\gamma_i(x)}
  <
  \tau_K(x),
  \label{eq:single-domain-margin-condition-current}
\end{equation}
then all top-$K$ experts lie in $\mathcal L(u_r)$ and
\begin{equation}
  N_r(x)=1,
  \qquad
  \operatorname{Comm}_r(x)=s_{\rm tok}c_r.
  \label{eq:single-domain-communication-current}
\end{equation}
\end{theorem}

If the first $r$ score margins satisfy $\gamma_i(x)\ge\gamma$, then
\begin{equation}
  \varepsilon_r(x)\le re^{-\gamma}.
  \label{eq:uniform-margin-prefix-bound-current}
\end{equation}
Thus $\gamma>\log(r/\tau_K(x))$ is sufficient for one-domain dispatch at
the depth-$r$ boundary.

\paragraph{Interpretation.}
An upper-level decision acts on an entire subtree: a large learned margin
suppresses the total probability of every leaf in the nonpreferred branch, so
after $r$ confident decisions the routing mass is concentrated in a subtree of
$E/2^r$ experts; if that subtree is placed within one hardware domain, the
lower levels refine the expert choice locally. This per-token result is
compatible with the population-level anti-collapse result of
\cref{sec:score-anticollapse}: arrival-weighted anchoring prevents the token
population as a whole from placing arbitrarily little routing mass on either
child of an internal node, but does not require each token to split its mass
evenly. A token may assign most of its probability to one branch while other
tokens favor the opposite branch, so BRANCH-MoE can maintain broad expert
utilization \emph{across tokens} while producing concentrated, structured
assignments \emph{for each token}.

\section{Experiments}
\label{sec:experiments}

We evaluate \branch{} against representative MoE routers on four benchmarks,
testing (i) whether it is competitive on task quality without an auxiliary
balancing loss and (ii) whether it induces the expert locality analyzed in
\cref{sec:communication-locality}. In short, quality is at parity with the
strongest baselines on all four benchmarks, locality appears on the three UCI
benchmarks but not on Criteo, and load stays balanced without an auxiliary loss.

\subsection{Setup}
\label{sec:setup}

We use Criteo (${\approx}41.5$M examples, $39$ fields), Forest Covertype ($581$K,
$54$ features), HIGGS ($600$K, $28$), and YearPredictionMSD ($515$K, $90$). All
MoE conditions share one block of $E = 16$ SwiGLU experts (top-$4$, capacity
factor $1.25$) between a minimal input stage and a linear head, and differ only
in the gate. The baselines are softmax top-$K$ with an auxiliary loss (Switch),
dynamic-bias routing (DeepSeek-V3), logit-normalized softmax with an auxiliary
loss (Skywork), and a hash-based control (Hash). Node maps are linear or two-layer MLPs
of width $h \in \{2, 4, 8\}$. Every condition has five seeds; comparisons are
paired by seed on UCI and unpaired (Welch) on Criteo, where TPU execution is not
bitwise reproducible. We report the normalized tree distance $L/D$ between
co-selected experts ($L = \mathrm{bit\_length}(i \oplus j)$, $D = \log_2 E$;
random pairs score $0.817$, lower is more local) and the load CV (std/mean of
per-expert load). Details: \cref{sec:app-setup}.

\begin{table}[t]
\centering
\footnotesize
\setlength{\tabcolsep}{4pt}
\caption{Main results ($E{=}16$, top-$4$, $5$ seeds; mean\,$\pm$\,sd).
(a)~Best held-out loss, each router at its best of four node maps (tag),
selected on the reporting split; \branch{} also with the linear map covered by
the theory. Bold: best arm. Hash ignores the node map (UCI: best of its four
identical runs).
(b)~Normalized tree distance (random pairs $0.817$; bold: most local on UCI) /
load CV of the same arms. $^\dagger$Not read as locality (\cref{sec:criteo}).
Per-map results: \cref{sec:app-results}.}
\label{tab:main}
\begin{tabular}{@{}lcccc@{}}
\toprule
Method & Criteo (logloss) & Covertype (CE) & HIGGS (logloss) & MSD (MSE) \\
\midrule
\multicolumn{5}{@{}l}{\textit{(a) Best held-out loss} $\downarrow$} \\
\branch{} (best map) & $\mathbf{0.44596}$\tabsd{0.00015}\tabmap{h4} & $\mathbf{0.2372}$\tabsd{0.0013}\tabmap{h8} & $0.5256$\tabsd{0.0014}\tabmap{h8} & $0.6767$\tabsd{0.0027}\tabmap{h8} \\
\branch{} (linear)   & $0.44615$\tabsd{0.00025} & $0.2545$\tabsd{0.0041} & $0.5259$\tabsd{0.0008} & $0.6843$\tabsd{0.0035} \\
Skywork              & $0.44603$\tabsd{0.00030}\tabmap{lin} & $0.2406$\tabsd{0.0036}\tabmap{h8} & $\mathbf{0.5247}$\tabsd{0.0005}\tabmap{lin} & $0.6757$\tabsd{0.0015}\tabmap{h2} \\
DeepSeek-V3          & $0.44605$\tabsd{0.00005}\tabmap{h8} & $0.2421$\tabsd{0.0033}\tabmap{h8} & $0.5264$\tabsd{0.0013}\tabmap{lin} & $\mathbf{0.6746}$\tabsd{0.0043}\tabmap{h8} \\
Switch               & $0.44623$\tabsd{0.00019}\tabmap{lin} & $0.2406$\tabsd{0.0028}\tabmap{h8} & $0.5276$\tabsd{0.0003}\tabmap{h4} & $0.6766$\tabsd{0.0038}\tabmap{h8} \\
Hash                 & $0.44707$\tabsd{0.00028} & $0.4465$\tabsd{0.0011} & $0.5574$\tabsd{0.0007} & $0.6899$\tabsd{0.0024} \\
\midrule
\multicolumn{5}{@{}l}{\textit{(b) Tree distance $\downarrow$ / load CV $\downarrow$ of the same arms}} \\
\branch{} (best map) & $\mathbf{0.757}$ / $0.268$ & $0.742$ / $0.112$ & $0.736$ / $0.231$ & $0.749$ / $0.219$ \\
\branch{} (linear)   & $0.813$ / $0.115$ & $\mathbf{0.718}$ / $0.093$ & $\mathbf{0.710}$ / $0.227$ & $\mathbf{0.718}$ / $0.177$ \\
Skywork              & $0.818$ / $0.004$ & $0.817$ / $0.137$ & $0.812$ / $0.183$ & $0.813$ / $0.536$ \\
DeepSeek-V3          & $0.819$ / $0.004$ & $0.813$ / $0.001$ & $0.816$ / $0.002$ & $0.813$ / $0.003$ \\
Switch               & $0.821$ / $0.003$ & $0.819$ / $0.141$ & $0.814$ / $0.104$ & $0.800$ / $0.278$ \\
Hash                 & $0.817$ / $0.002$ & $0.817$ / $0.003$ & $0.817$ / $0.002$ & $0.817$ / $0.004$ \\
\bottomrule
\end{tabular}
\end{table}

\subsection{Results}
\label{sec:results}

\textbf{Quality.\;}
\branch{} is at parity with the strongest baselines (\cref{tab:main}a). On
Criteo the learned routers span $1.9 \times 10^{-4}$ logloss with the linear
map, about half the smallest resolvable difference; at its best map ($h = 4$)
\branch{} has the lowest mean, but its leads over Skywork ($p = 0.63$) and
DeepSeek-V3 ($p = 0.26$) are within seed noise, and DeepSeek-V3 has the higher
AUC. On Covertype \branch{} ($h = 8$) has the lowest mean cross-entropy
($0.2372$); its $0.0034$ margin over Switch is resolved only by the paired test
($t_4 = -2.86$, critical value $2.776$), the equal margin over Skywork by
neither test, and the $0.0049$ margin over DeepSeek-V3 by the paired test
($t_4 = -3.25$). \branch{} ranks second on HIGGS ($0.0009$ behind Skywork,
within seed spread) and fourth on MSD, where the four learned routers span
$0.0021$ against seed sds of $0.0015$--$0.0043$. Hash is the worst MoE arm
everywhere.

\medskip

\noindent\textbf{Locality.\;}
On the three UCI benchmarks \branch{} is the only router materially below random
(\cref{tab:main}b). With the linear map its tree distance is $0.710$--$0.718$,
$12$--$13\%$ below random ($18$--$23\%$ on index distance), and $0.736$--$0.749$
at $h = 8$; every flat router stays within $3.5\%$ of random (closest: Switch on
MSD, $0.800$). On Criteo \branch{} shows no locality: with the linear map its
tree distance ($0.813$) is indistinguishable from random ($p = 0.48$), although
an untrained tree scores $0.725$. The lower values with MLP maps ($^\dagger$)
also appear for the flat Skywork router ($0.75$--$0.77$), so we do not read them
as locality. Criteo differs from UCI in input representation, batch size, and
training length; we have not isolated which matters (\cref{sec:criteo}).

\medskip

\noindent\textbf{Balance.\;}
\branch{} uses no auxiliary loss or bias controller. On UCI its linear-map load
CV ($0.09$--$0.23$) is comparable to that of the auxiliary-loss routers at their
best maps (Switch $0.10$--$0.28$, Skywork $0.14$--$0.54$); on Covertype the only
dead experts are Skywork's at $h = 2$ ($6$ of $80$ slots). On Criteo its CV of
$0.115$ exceeds the flat routers' ($\le 0.006$) but costs nothing measurable: it
drops $0.08\%$ of dispatches.

\begin{figure}[t]
\centering
\begin{tikzpicture}
\begin{axis}[
  name=left, width=0.53\linewidth, height=3.0cm,
  symbolic x coords={lin,h2,h4,h8}, xtick=data,
  ylabel={Covertype CE}, ylabel style={font=\scriptsize, yshift=-4pt},
  tick label style={font=\scriptsize}, title={\scriptsize (a) loss vs.\ node map},
  title style={yshift=-6pt},
  legend style={font=\tiny, at={(0.98,0.98)}, anchor=north east, draw=none, fill=none, row sep=-2pt,
    /tikz/every even column/.append style={column sep=4pt}}, legend columns=2,
  legend cell align=left, mark size=1.3pt, ymin=0.232, ymax=0.300]
\addplot+[thick, error bars/.cd, y dir=both, y explicit] coordinates {(lin,0.25451)+-(0,0.00408) (h2,0.25206)+-(0,0.00487) (h4,0.24394)+-(0,0.00356) (h8,0.23720)+-(0,0.00125)};
\addplot+[error bars/.cd, y dir=both, y explicit] coordinates {(lin,0.25862)+-(0,0.00365) (h2,0.25169)+-(0,0.00324) (h4,0.24277)+-(0,0.00494) (h8,0.24062)+-(0,0.00358)};
\addplot+[error bars/.cd, y dir=both, y explicit] coordinates {(lin,0.26150)+-(0,0.00136) (h2,0.25646)+-(0,0.00713) (h4,0.24286)+-(0,0.00180) (h8,0.24062)+-(0,0.00281)};
\addplot+[error bars/.cd, y dir=both, y explicit] coordinates {(lin,0.27032)+-(0,0.00225) (h2,0.26620)+-(0,0.00301) (h4,0.25417)+-(0,0.00196) (h8,0.24207)+-(0,0.00326)};
\legend{\branch{},Skywork,Switch,DeepSeek-V3}
\end{axis}
\begin{axis}[
  at={(left.east)}, anchor=west, xshift=1.1cm,
  width=0.53\linewidth, height=3.0cm,
  symbolic x coords={lin,h2,h4,h8}, xtick=data,
  ylabel={\branch{} tree dist.}, ylabel style={font=\scriptsize, yshift=-4pt},
  tick label style={font=\scriptsize}, title={\scriptsize (b) locality vs.\ node map},
  title style={yshift=-6pt},
  legend style={font=\tiny, at={(0.5,0.03)}, anchor=south, draw=none, fill=none,
    /tikz/every even column/.append style={column sep=4pt}}, legend columns=3,
  legend cell align=left, mark size=1.3pt, ymin=0.64, ymax=0.83]
\addplot[fill=gray!20, draw=none, forget plot] coordinates {(lin,0.800) (h8,0.800) (h8,0.819) (lin,0.819)} -- cycle;
\addplot[dashed, gray, forget plot] coordinates {(lin,0.817) (h8,0.817)};
\addplot+[thick] coordinates {(lin,0.718) (h2,0.746) (h4,0.748) (h8,0.742)};
\addplot+[] coordinates {(lin,0.710) (h2,0.750) (h4,0.774) (h8,0.736)};
\addplot+[] coordinates {(lin,0.718) (h2,0.746) (h4,0.746) (h8,0.749)};
\legend{Covertype,HIGGS,MSD}
\end{axis}
\end{tikzpicture}
\caption{Node-map effect. (a)~Covertype loss of the learned routers vs.\ node
map, $\pm$sd (Hash, at $0.447$, is off-scale). (b)~\branch{}'s tree
distance vs.\ node map on UCI; dashed: random pairs ($0.817$); shaded: flat
routers at their best maps.}
\label{fig:nodemap}
\end{figure}
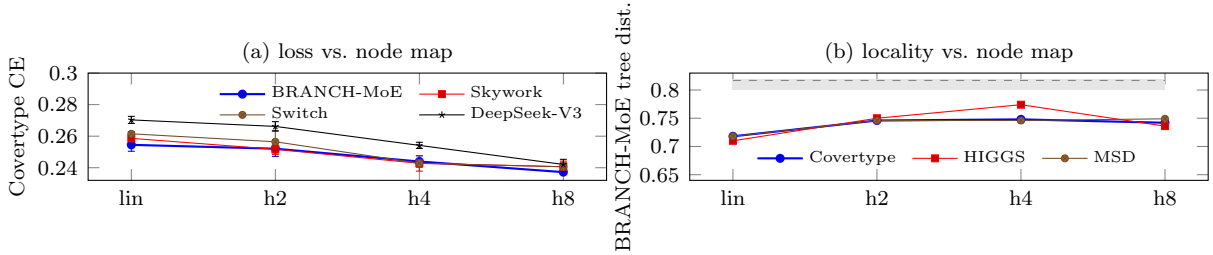

\medskip

\noindent\textbf{Node maps.\;}
The node map can matter more than the routing rule (\cref{fig:nodemap}a). On
Covertype, replacing the linear map by an $h = 8$ MLP lowers cross-entropy by
$0.017$--$0.028$ for every learned router (paired $t$ from $-7.99$ to $-21.44$),
$3.5$--$5.8\times$ the largest gap between routers at $h = 8$ ($0.0049$; $0.016$
at linear), while Hash, which ignores the node map, moves by only $0.00013$. The
sign depends on the dataset: MSD agrees; on HIGGS the linear map is best for
Skywork and DeepSeek-V3 and ties for \branch{}; on Criteo no router's loss
changes significantly ($\le 4.1 \times 10^{-4}$, $p > 0.1$) except Switch's,
which fails with every MLP map (${\approx}68\%$ of dispatches dropped). The MLP
also loosens \branch{}'s structure: on UCI its tree distance rises by
$0.024$--$0.031$ (\cref{fig:nodemap}b; same sign on all three datasets, not
resolved on Covertype alone), and on Criteo its load CV rises from $0.115$ to
$0.21$--$0.38$ ($p \le 0.01$ at $h = 4$ and $8$).

\medskip

\noindent\textbf{Takeaway.\;}
Balance without an auxiliary loss and locality on UCI are supported, and quality
is at parity; locality on Criteo and whether MLP maps erode locality remain open
(\cref{sec:app-results}).

\medskip

\noindent\textbf{Conclusion and limitations.}
BRANCH-MoE combines arrival-weighted score anchoring with hierarchical routing to encourage balanced population-level utilization and structured per-token expert selection without an auxiliary balancing loss. Our analysis characterizes anchor tracking, establishes soft-mass anti-collapse for linear routing under log-concave arrivals, relates realized expert-selection frequency to optimization progress, and bounds communication locality under prefix-aligned placement. The experiments demonstrate competitive task quality and structured locality across several benchmarks. The present study is limited to a single BRANCH-MoE layer and therefore does not capture interactions among routing decisions across multiple layers. Future work will extend both the theory and empirical evaluation to multilayer architectures and develop fused, topology-aware implementations for large-scale distributed settings.


\newpage
\subsection*{AI use statement}
In this work, we used generative AI tools to improve the manuscript’s clarity, presentation and writing. We have not used generative AI tools for conducting experiments.
All AI-assisted material was carefully reviewed and verified by the authors. We accept full responsibility for the final content of this work, including all text, claims, and artifacts produced with the assistance of generative AI.
\bibliography{iclr2027_conference}
\bibliographystyle{iclr2027_conference}

\newpage
\appendix







\section{Proofs}
\label{app:proofs}

\begin{proof}[Proof of \cref{thm:anchor-ema-stationary}]
Let $e_t:=\nu_t-\bar\nu$. From
the constant-decay update and \cref{eq:anchor-noise},
\[
  e_{t+1}=\alpha e_t+(1-\alpha)\xi_t.
\]
Taking conditional expectation given $\mathcal F_t$ gives
$\E[e_{t+1}\mid\mathcal F_t]=\alpha e_t$. Iterating proves
\cref{eq:anchor-ema-mean}.

For the second moment,
\begin{align*}
  \E[e_{t+1}^2\mid\mathcal F_t]
  &=
  \alpha^2e_t^2
  +2\alpha(1-\alpha)e_t\E[\xi_t\mid\mathcal F_t]
  +(1-\alpha)^2\E[\xi_t^2\mid\mathcal F_t]\\
  &\le
  \alpha^2e_t^2+(1-\alpha)^2s^2.
\end{align*}
Taking expectations and unrolling the recursion yields
\[
  \E[e_t^2]
  \le
  \alpha^{2t}e_0^2
  +(1-\alpha)^2s^2\sum_{k=0}^{t-1}\alpha^{2k}.
\]
Using
$\sum_{k=0}^{t-1}\alpha^{2k}=(1-\alpha^{2t})/(1-\alpha^2)$
gives \cref{eq:anchor-ema-mse}.
\end{proof}

\begin{proof}[Proof of \cref{prop:anchor-ema-drift}]
Write $\xi_t:=\widehat\nu_t-\bar\nu_t$ and
$e_t:=\nu_t-\bar\nu_t$. Then
\[
  e_{t+1}
  =
  \alpha e_t
  +(\bar\nu_t-\bar\nu_{t+1})
  +(1-\alpha)\xi_t.
\]
Unrolling gives
\begin{align*}
  e_t
  ={}&
  \alpha^t e_0
  +\sum_{k=0}^{t-1}
    \alpha^{t-1-k}(\bar\nu_k-\bar\nu_{k+1})\\
  &+(1-\alpha)
  \sum_{k=0}^{t-1}\alpha^{t-1-k}\xi_k.
\end{align*}
The absolute value of the deterministic first line is at most
\[
  \alpha^t|e_0|+
  v\sum_{k=0}^{t-1}\alpha^{t-1-k}
  =
  \alpha^t|e_0|+\frac{v(1-\alpha^t)}{1-\alpha}.
\]
The martingale-difference property makes distinct noise terms orthogonal and
eliminates their cross term with the deterministic part. Hence the noise
variance is at most
\[
  (1-\alpha)^2s^2
  \sum_{k=0}^{t-1}\alpha^{2k}
  =
  \frac{1-\alpha}{1+\alpha}
  (1-\alpha^{2t})s^2.
\]
Combining the two bounds proves \cref{eq:anchor-drift-bound}.
\end{proof}

\begin{proof}[Proof of \cref{lem:soft-logconcavity}]
Along the path to node $(d,j)$, its mask can be written as
\[
  m_j^{(d)}(x)
  =
  \prod_{u\in\pathset((d,j))}
  \sigma(s_u z_u(x)),
\]
where $s_u\in\{-1,+1\}$ selects the appropriate branch and $z_u(x)$ is an
affine centered score. The function
$y\mapsto\log\sigma(y)=-\log(1+\exp(-y))$ is concave. Therefore each factor
$x\mapsto\sigma(s_u z_u(x))$ is log-concave, and so is their product. The
unnormalized arrival density is the product of this mask and the log-concave
input density, hence is log-concave. Normalization preserves log-concavity.
\end{proof}

\begin{proof}[Proof of \cref{cor:linear-soft-anticollapse}]
Let $\mu_j^{(d)}:=\E_{\mathcal D_j^{(d)}}[X]$. Linearity gives
\[
  \nu_j^{(d),\star}
  =
  \bigl(W_j^{(d)}\bigr)^\top\mu_j^{(d)},
\]
so the population-centered score is
\[
  Z
  =
  \bigl(W_j^{(d)}\bigr)^\top
  (X-\mu_j^{(d)}).
\]
Gr\"unbaum's inequality for log-concave measures implies
\[
  \Pp(Z\ge0)\ge\frac1{\mathrm e},
  \qquad
  \Pp(Z\le0)\ge\frac1{\mathrm e}.
\]
At the exact population anchor,
$\pi_{j,+}^{(d)}=\E[\sigma(Z)]$ and
$\pi_{j,-}^{(d)}=\E[\sigma(-Z)]$, where both expectations are under
$\mathcal D_j^{(d)}$. Since $\sigma(y)\ge1/2$ for $y\ge0$, each expectation
is at least $1/(2\mathrm e)$, proving
\cref{eq:linear-soft-anticollapse}.

For an arbitrary active anchor, the two expectations become
$\E[\sigma(Z-\delta_j^{(d)})]$ and
$\E[\sigma(\delta_j^{(d)}-Z)]$. They are bounded below by one half of
$\Pp(Z\ge\delta_j^{(d)})$ and
$\Pp(Z\le\delta_j^{(d)})$, respectively. If the density of $Z$ is bounded
by $A_j^{(d)}$, moving the threshold from zero to
$\delta_j^{(d)}$ changes either tail probability by at most
$A_j^{(d)}|\delta_j^{(d)}|$. Combining this fact with Gr\"unbaum's bounds
and truncating the resulting lower bound at zero proves
\cref{eq:linear-soft-anticollapse-error}.
\end{proof}

\begin{proof}[Proof of \cref{thm:expert-progress-current}]
Write $u_t=\theta_{e,t}-\theta_e^\star$. Conditional on $\mathcal H_t$,
unbiasedness and \cref{eq:expert-sgd-update-current} imply
\begin{align*}
  \E[\|u_{t+1}\|^2\mid\mathcal H_t]
  =&\;\|u_t\|^2
  -2p_e\gamma_e
    \langle u_t,\nabla\mathcal L_e(\theta_{e,t})\rangle\\
  &+p_e\gamma_e^2
    \E[\|g_{e,t}\|^2\mid\mathcal H_t,I_{e,t}=1].
\end{align*}
Strong convexity gives
$\langle u_t,\nabla\mathcal L_e(\theta_{e,t})\rangle
\ge\mu_e\|u_t\|^2$. Smoothness, optimality of $\theta_e^\star$, and
\cref{eq:expert-gradient-assumption-current} give
\[
  \E[\|g_{e,t}\|^2\mid\mathcal H_t,I_{e,t}=1]
  \le\|\nabla\mathcal L_e(\theta_{e,t})\|^2+\sigma_e^2.
\]
Because a differentiable convex $L_e$-smooth function has an
$L_e^{-1}$-cocoercive gradient,
$\|\nabla\mathcal L_e(\theta_{e,t})\|^2
\le L_e\langle u_t,\nabla\mathcal L_e(\theta_{e,t})\rangle$.
Therefore, using $\gamma_e\le1/L_e$,
\[
  \E[\|u_{t+1}\|^2\mid\mathcal H_t]
  \le(1-p_e\gamma_e\mu_e)\|u_t\|^2
  +p_e\gamma_e^2\sigma_e^2.
\]
Iteration of this recursion and summation of the geometric series yield
\[
  \E\|u_T\|^2
  \le(1-p_e\gamma_e\mu_e)^T\|u_0\|^2
  +\frac{\gamma_e\sigma_e^2}{\mu_e}.
\]
Finally, $1-x\le\exp(-x)$ gives the second line of
\cref{eq:expert-convergence-current}. If $p_e\ge p_{\min}$, replacing $p_e$
by $p_{\min}$ in the exponential transient and solving for $T$ gives
\cref{eq:expert-iteration-complexity}.
\end{proof}
\begin{proof}[Proof of \cref{thm:learned-margin-communication-current}]
The leaf masks below $u_r=(r,j_r)$ sum to its incoming mask:
\begin{equation}
  \sum_{e\in\mathcal L(u_r)}m_{e,t}^{(D)}(x)
  =
  m_{j_r,t}^{(r)}(x).
  \label{eq:subtree-mass-conservation-current}
\end{equation}
This follows by induction from the cascade, because the two children of a
node with incoming mask $m$ receive masses $m(1-p)$ and $mp$, whose sum is
$m$.

Along the preferred path, the probability assigned at level $i$ is
$\sigma(\gamma_i(x))$. Therefore,
\[
  \sum_{e\in\mathcal L(u_r)}m_{e,t}^{(D)}(x)
  =
  m_{j_r,t}^{(r)}(x)
  =
  \prod_{i=0}^{r-1}\sigma(\gamma_i(x)).
\]
Since the full leaf mask sums to one,
\[
  \varepsilon_r(x)
  =
  1-\prod_{i=0}^{r-1}\sigma(\gamma_i(x)).
\]
For numbers $a_i\in[0,1]$,
$1-\prod_i(1-a_i)\le\sum_i a_i$. Taking
$a_i=\sigma(-\gamma_i(x))$ gives
\[
  \varepsilon_r(x)
  \le
  \sum_{i=0}^{r-1}\sigma(-\gamma_i(x))
  \le
  \sum_{i=0}^{r-1}e^{-\gamma_i(x)},
\]
which proves \cref{eq:prefix-mass-concentration-current}.

Let
\[
  h_r(x)
  :=
  |\mathcal S_K(x)\setminus\mathcal L(u_r)|.
\]
Every selected leaf has probability at least $\tau_K(x)$, and hence
\[
  \varepsilon_r(x)
  \ge
  \sum_{e\in\mathcal S_K(x)\setminus\mathcal L(u_r)}
  m_{e,t}^{(D)}(x)
  \ge
  h_r(x)\tau_K(x).
\]
Because $h_r(x)$ is integer, this proves
\cref{eq:number-selected-outside-current}.

All selected experts inside $\mathcal L(u_r)$ occupy one domain, and every
selected expert outside it can add at most one further domain. Thus
$N_r(x)\le1+h_r(x)$, while trivially $N_r(x)\le K$. This proves
\cref{eq:communication-domain-bound-current}.

Finally, \cref{eq:single-domain-margin-condition-current} implies
$\varepsilon_r(x)<\tau_K(x)$. If a selected expert lay outside
$\mathcal L(u_r)$, its probability would be at least $\tau_K(x)$, which
would contradict the definition of $\varepsilon_r(x)$. Hence every selected
expert lies in the same depth-$r$ domain.
\end{proof}

\section{Experimental Details}
\label{sec:app-setup}

\subsection{Datasets}
\label{sec:app-datasets}
\Cref{tab:datasets} summarizes the four benchmarks. Criteo is the primary
large-scale task; the three UCI datasets are tabular benchmarks that admit
many-seed replication at modest cost.

\begin{table}[t]
\centering
\small
\caption{Benchmark suite. All MoE conditions use $E = 16$ experts with
top-$4$ routing. ``Examples'' is the full corpus size (train $+$ held-out);
``Seeds'' is the number of independent replications per condition.}
\label{tab:datasets}
\begin{tabular}{llrrrr}
\toprule
Dataset & Task & Examples & Features & Seeds & Metric \\
\midrule
Criteo (small)       & CTR (binary)        & ${\approx}41.5$M\footnotemark  & $39$ fields & $5$ & Logloss / AUC \\
Forest Covertype     & $7$-class           & $581$K  & $54$        & $5$ & Cross-entropy / Acc.\ \\
HIGGS                & Binary              & $600$K  & $28$        & $5$ & Logloss / AUC \\
YearPredictionMSD    & Regression          & $515$K  & $90$        & $5$ & MSE / RMSE (yr) \\
\bottomrule
\end{tabular}
\end{table}
\footnotetext{The held-out split is exactly $4.42$M examples (one full
evaluation pass of $2{,}158$ steps at batch $2048$). The training split is not
counted directly: we infer ${\approx}37$M from the ratio of on-disk sizes
($5.93$\,GB train against $0.71$\,GB held-out), which assumes a common mean
record size across the two splits. All reported results depend on the training
\emph{step} count, not on this figure, so the corpus size should be read as
approximate.}

\paragraph{Criteo.} Each example comprises $13$ integer (count) features and
$26$ categorical features, for $39$ fields in total, with a binary click label.
Categorical fields are mapped to integer IDs by a string lookup built from a
vocabulary with a minimum-frequency cutoff of $100$; post-cutoff vocabulary
sizes $V$ range from $4$ to $20{,}136$. Integer fields are discretized into $41$
bins by the boundaries $b_i = 1.5^{\,i-1} - 0.49$ for $i = 0,\dots,39$,
converting magnitude into an ordered categorical ID. Embeddings are not shared
across fields: integer fields use width $10$, categorical fields use
$d_f = \min(\lfloor 6V_f^{1/4}\rfloor, V_f) \in [4, 71]$, and the concatenation
yields a $1137$-dimensional input.

\paragraph{UCI benchmarks.} Forest Covertype ($54$ cartographic features, $7$
cover classes, $80/20$ split) and HIGGS ($28$ particle-physics features, the
first $600$K of the $11$M-event corpus, $500$K/$100$K split) are standardized.
YearPredictionMSD ($90$ timbre statistics) uses the canonical
$463{,}715/51{,}630$ split, which prevents artist leakage; its features are
likewise standardized, the target is the standardized release year, and we
report unstandardized RMSE in years alongside the MSE objective. HIGGS and MSD
inputs are additionally clipped at $\pm 10\sigma$; Covertype inputs are not.

On Forest Covertype one preprocessing detail proves critical: $40$ of the $54$
features are binary soil-type indicators, the rarest firing on only $2$ of the
$464{,}809$ training rows. Standard $z$-scoring with an $\epsilon = 10^{-7}$
denominator guard maps a column that is constant in the sample used to fit the
normalizer to spikes of magnitude $\sim 10^{7}$ on the rows where it does fire,
which corrupts running statistics and induces bimodal expert
collapse in tree routers. Pinning any column with adapted variance below
$10^{-4}$ to unit variance eliminates the collapse and reduces the seed-to-seed
accuracy spread of \branch{} from $6.2$ to $0.46$ percentage points.

\subsection{Baselines}
\label{sec:app-baselines}
We benchmark against four representative MoE
routing paradigms, covering auxiliary-loss, auxiliary-loss-free, normalization-based,
and deterministic mechanisms.
\begin{enumerate}[leftmargin=*]
    \item \textbf{Top-$K$ Softmax Routing (Switch).} The canonical learned
    paradigm~\citep{fedus2022switch}: a learned gating matrix, a softmax, and
    top-$K$ dispatch, with multiplicative uniform jitter
    $[1-\epsilon, 1+\epsilon]$, $\epsilon = 10^{-2}$, during training. Load
    balance is enforced by the auxiliary loss
    $\mathcal{L}_{\text{balance}} = \alpha E \sum_e f_e P_e$ with
    $\alpha = 0.01$.
    \item \textbf{Auxiliary-Loss-Free Routing (DeepSeek-V3).} Affinities are
    computed with an independent sigmoid per expert, so experts do not compete
    for probability mass~\citep{deepseekai2024deepseekv3}. A non-trainable
    additive bias $b_e$ enters top-$K$ selection only and is updated after each
    step by a fixed increment $\gamma = 10^{-3}$ in the direction that relieves
    congestion; the combining gates remain the unbiased affinities.
    \item \textbf{Gating Logit Normalization (Skywork).} Logits are
    standardized across the expert dimension before the softmax, which
    counteracts representation collapse and sharpens the gate
    distribution~\citep{wei2024skywork}. A fixed auxiliary balancing
    coefficient is retained. We include this baseline because it proves to be
    the strongest competitor on two of the four benchmarks and ties Switch as
    the runner-up on a third.
    \item \textbf{Deterministic Balanced Routing (Hash).} A fixed assignment
    mapping each example to experts by a $32$-bit avalanche hash of its feature
    vector, following Hash Layers~\citep{roller2021hash}. The assignment is
    uniform by construction but carries no dependence on content beyond the
    hash, isolating the value of \emph{learned} gating against a perfectly
    balanced but semantically uninformed one, at zero router parameter cost.
    \item \textbf{\branch{} (ours).} The $E$ experts are the leaves of a
    balanced binary tree of depth $\log_2 E$. At internal node $n$ the
    probability of branching right is
    $\sigma(f_n(\mathbf{x}) - \mu_n)$. During training $\mu_n$ is the
    arrival-weighted mean score of the current batch at $n$; at inference it is
    $\nu_n$, a non-trainable exponential moving average (decay $0.99$) of that
    batch mean, updated at every training step. Centering each split on the
    mean score of its own traffic places the decision boundary through the
    middle of the distribution actually arriving there, keeping both children
    populated without a balancing penalty. Leaf scores are products of branch
    probabilities along each root-to-leaf path; all $E - 1$ node scores are
    evaluated, and the products are formed in $\log_2 E$ vectorized steps.
\end{enumerate}

\subsection{Node Maps, Architecture, and Protocol}
\label{sec:app-protocol}

\paragraph{Node maps.}
\label{sec:nodemap}
The score $f_n$ at each node, and correspondingly the logit map of each flat
router, is one of two forms, swept as a factor in its own right:
\begin{itemize}[leftmargin=*]
    \item \textbf{Linear:} $f_n(\mathbf{x}) = \mathbf{w}_n^{\top}\mathbf{x}$.
    \item \textbf{MLP:} a two-layer network of hidden width
    $h \in \{2, 4, 8\}$ per node, $f_n(\mathbf{x}) = \mathbf{v}_n^{\top}\mathrm{ReLU}(\mathbf{U}_n\mathbf{x} + \mathbf{c}_n)$.
\end{itemize}
This factor is swept across all four learned routers and the Hash control on
the three UCI benchmarks, and across all four learned routers on Criteo.


\paragraph{Architecture and protocol.}
\label{sec:protocol}

 To isolate routing, we adopt a minimal backbone: an input stage, a single MoE block, and a linear head. On Criteo, the input stage projects the $39$ embeddings to $256$ dimensions, followed by SiLU and layer normalization; on the UCI datasets, it is a feature-standardization layer alone with no projection. The MoE block has no residual connection. Within this block, each expert is a SwiGLU network defined as $\mathrm{down}(\mathrm{SiLU}(\mathrm{gate}(\mathbf{x})) \odot \mathrm{up}(\mathbf{x}))$, where the $\mathrm{up}$ layer expands the input into a higher-dimensional space and the $\mathrm{down}$ layer compresses it back to the baseline hidden size. Finally, a post-MoE layer normalization precedes the head on Criteo, HIGGS, and MSD (the Covertype model has none).

All conditions use $E = 16$
and top-$4$ routing with capacity factor $1.25$, and differ only
in the gate $g(\cdot)$, so comparisons are at matched capacity. Criteo trains
for $37{,}500$ steps at batch $2048$ with Adam ($\text{lr} = 5\times 10^{-4}$),
organized as $15$ evaluation intervals of $2{,}500$ steps; since an interval
covers $5.12$M examples against a ${\approx}37$M training split, this is
${\approx}2.1$ passes over the data rather than $15$. The UCI datasets train
for $20$ true epochs with Adam ($\text{lr} = 10^{-3}$) at batch $512$
(Covertype) or $256$ (HIGGS, MSD). Every condition is replicated over $5$ seeds. On the
UCI benchmarks the seed governs initialization but not the data order, which
is shuffled without a seed; all arms within a dataset share the identical seed
set, so comparisons are \emph{paired} by seed. On Criteo the seed likewise
governs initialization only, and TPU execution is not bitwise
reproducible, so Criteo comparisons are unpaired (two-sided Welch $t$-tests).
Re-running an identical Criteo configuration with the same seed changes
held-out logloss about as much as changing the seed does (standard deviation
$1.5$--$2.8 \times 10^{-4}$ within a seed against $1.9$--$3.0 \times 10^{-4}$
across seeds).

\subsection{Locality Metrics}
\label{sec:locality-metrics}
For each token routed to experts $(i, j)$ we measure two index-based distances,
averaged over all $\binom{k}{2}$ co-selected pairs.
\begin{enumerate}[leftmargin=*]
    \item \textbf{Tree distance $L$.} Reading an expert index as a $D$-bit
    root-to-leaf path ($D = \log_2 E$), the number of levels to the lowest
    common ancestor is
    $L(i,j) = \mathrm{bit\_length}(i \oplus j) \in \{1, \dots, D\}$. Under block
    placement across $P = 2^m$ devices, two experts share a device if and only
    if $L(i,j) \le D - m$, making $L$ an exact topological proxy for
    cross-device fanout.
    \item \textbf{Index gap $|i - j|$.} The linear gap between expert indices.
\end{enumerate}
Both require a reference: for uniformly random \emph{distinct} pairs,
$\E[L] = (D - 1 + 2^{-D})/(1 - 2^{-D})$ and $\E[|i-j|] = (E+1)/3$, giving
$3.2667$ and $5.6667$ at $E = 16$ ($D = 4$). Normalized by $D$ and $E - 1$ these
are $0.8167$ and $0.3778$ --- random routing already scores $0.82$ on the tree
metric, so locality is only evidenced \emph{below} these values, and an
unreferenced number is uninterpretable. We also report the load coefficient of
variation $\mathrm{CV} = \mathrm{std}(f)/\mathrm{mean}(f)$ of the per-expert
dispatch distribution.

Tree distance is also the more sensitive of the two index-based measures: it
separates power-of-two boundary pairs such as $(7, 8)$, which have
$|i - j| = 1$ but lie at maximal tree separation, and which the index gap
therefore scores as maximally local when they are maximally remote.

\section{Additional Results}
\label{sec:app-results}

\subsection{Criteo: Click-Through-Rate Prediction}
\label{sec:criteo}

\begin{table}[t]
\centering
\small
\caption{Criteo, $E = 16$, top-$4$, capacity factor $1.25$, $5$ seeds (mean
$\pm$ std). Logloss is binary cross-entropy; logloss and AUC are the best
held-out values over the $15$ evaluations. Load CV and the normalized tree and
index distances are training-time routing statistics; random baselines are
$0.8167$ and $0.3778$. Upper block: every router with the linear node map.
Lower block: the two routers whose lowest-loss node map is an MLP, at that map
(for Skywork and Switch it is the linear map). Bold marks the best logloss and
AUC. At $5$ vs.\ $5$ seeds, differences below ${\approx}3.7 \times 10^{-4}$ in
logloss or ${\approx}2.7 \times 10^{-4}$ in AUC are not resolvable (Welch
$t$-test, $p < 0.05$).}
\label{tab:criteo}
\begin{tabular}{lccccc}
\toprule
Method & Best val.\ logloss $\downarrow$ & Best val.\ AUC $\uparrow$ & Load CV $\downarrow$ & Tree $\downarrow$ & Index $\downarrow$ \\
\midrule
\multicolumn{6}{l}{\emph{Linear node map}} \\
Skywork          & $\mathbf{0.446034}$ & $\mathbf{0.805827}$ & $0.0040$ & $0.8184$ & $\mathbf{0.3778}$ \\
                 & $\pm 0.000298$ & $\pm 0.000184$ & & & \\
DeepSeek-V3      & $0.446082$ & $0.805799$ & $0.0057$ & $0.8188$ & $0.3849$ \\
                 & $\pm 0.000270$ & $\pm 0.000174$ & & & \\
\branch{} (ours) & $0.446149$ & $0.805707$ & $0.1152$ & $\mathbf{0.8126}$ & $0.3843$ \\
                 & $\pm 0.000250$ & $\pm 0.000175$ & & & \\
Switch softmax   & $0.446226$ & $0.805595$ & $0.0025$ & $0.8212$ & $0.3812$ \\
                 & $\pm 0.000185$ & $\pm 0.000204$ & & & \\
Hash             & $0.447065$ & $0.804716$ & $\mathbf{0.0023}$ & $0.8167$ & $\mathbf{0.3778}$ \\
                 & $\pm 0.000277$ & $\pm 0.000116$ & & & \\
\midrule
\multicolumn{6}{l}{\emph{Best MLP node map}} \\
\branch{} ($h{=}4$)   & $\mathbf{0.445959}$ & $\mathbf{0.805865}$ & $0.2676$ & $0.7572$ & $0.3255$ \\
                      & $\pm 0.000153$ & $\pm 0.000022$ & & & \\
DeepSeek-V3 ($h{=}8$) & $0.446049$ & $0.805970$ & $0.0038$ & $0.8190$ & $0.3824$ \\
                      & $\pm 0.000047$ & $\pm 0.000100$ & & & \\
\bottomrule
\end{tabular}
\end{table}

\Cref{tab:criteo} reports the benchmark. With the common linear node map, the
four learned routers are statistically tied: their mean held-out logloss spans
$1.9 \times 10^{-4}$, about half of the smallest resolvable difference, and
\branch{} ranks third, $1.1 \times 10^{-4}$ behind Skywork (Welch $p = 0.53$).
When each router uses its best node map, \branch{} with $h = 4$ has the lowest
mean logloss, but its margins over Skywork ($7.5 \times 10^{-5}$, $p = 0.63$)
and DeepSeek-V3 at $h = 8$ ($9.0 \times 10^{-5}$, $p = 0.26$) are well inside
seed noise; choosing the best of four node maps after the fact also favours
it, and DeepSeek-V3 at $h = 8$ has the higher AUC ($+1.05 \times 10^{-4}$,
$p = 0.08$). We therefore claim parity, not superiority. The only unambiguous
gap is to Hash, which trails every other row of the table by
$0.0008$--$0.0011$ in logloss and $0.0009$--$0.0013$ in AUC (Welch
$p < 0.002$ in every comparison). All Criteo tests are unpaired, because TPU execution is not
bitwise reproducible (\cref{sec:protocol}).

\paragraph{Balance.}
\branch{} keeps every expert in use without an auxiliary loss or bias
controller, but balances load less tightly than the flat routers: with the
linear node map its load CV is $0.115$ (largest and smallest expert shares
$1.17\times$ and $0.72\times$ the mean), against at most $0.006$ for every flat
router. At capacity factor $1.25$ this has no measurable cost: every
linear-map condition drops at most $0.11\%$ of dispatches at capacity, and
\branch{} drops $0.08\%$.

\paragraph{Locality does not appear on Criteo.}
With the linear node map, \branch{}'s normalized tree distance is
$0.813 \pm 0.012$, indistinguishable from the $0.8167$ random baseline
(one-sample $t$-test, $p = 0.48$), and its index distance ($0.384$) is
marginally \emph{above} random ($p = 0.04$). The tree does not start at
random: an untrained \branch{} scores $0.725$, and training moves it to the
random value. The flat routers stay within $2\%$ of both random baselines with
linear node maps, as expected for arbitrary expert labels. Lower tree
distances do appear with MLP node maps ($0.757$ for \branch{} at $h = 4$), but
they vary strongly across seeds, are lowest in the most imbalanced runs, and
appear also for the flat Skywork router ($0.75$--$0.77$, $p \le 0.05$ at every
width), whose expert indices carry no meaning. We have not diagnosed the
Skywork effect; it points to an index-order artifact in expert selection, such
as tie-breaking among exactly equal logits, and we read neither result as
locality. Criteo is therefore the one benchmark on which the locality of
\cref{sec:locality-results} does not replicate. It differs from the UCI
benchmarks in input representation (learned embeddings of categorical
fields), batch size ($2048$ against $256$--$512$), and training length
(${\approx}2.1$ passes over the data rather than $20$ epochs); we have not
isolated which of these is responsible.

\subsection{Forest Covertype: Multi-Class Classification}
\label{sec:covtype}

\begin{table}[t]
\centering
\small
\caption{Forest Covertype, $E = 16$, top-$4$, $20$ epochs, batch $512$,
$5$ paired seeds (mean $\pm$ std). $h$ is the MLP node-map hidden width
(\cref{sec:nodemap}). Hash ignores the node map by construction; its
four arms are the same configuration run four times and serve as an exact
negative control.}
\label{tab:covtype}
\begin{tabular}{lcccc}
\toprule
 & \multicolumn{4}{c}{Best val.\ cross-entropy $\downarrow$} \\
\cmidrule(lr){2-5}
Method & linear & $h{=}2$ & $h{=}4$ & $h{=}8$ \\
\midrule
\branch{} (ours) & $\mathbf{0.25451}$ & $0.25206$ & $0.24394$ & $\mathbf{0.23720}$ \\
                 & $\pm 0.00408$ & $\pm 0.00487$ & $\pm 0.00356$ & $\pm 0.00125$ \\
Skywork          & $0.25862$ & $\mathbf{0.25169}$ & $\mathbf{0.24277}$ & $0.24062$ \\
                 & $\pm 0.00365$ & $\pm 0.00324$ & $\pm 0.00494$ & $\pm 0.00358$ \\
Switch softmax   & $0.26150$ & $0.25646$ & $0.24286$ & $0.24062$ \\
                 & $\pm 0.00136$ & $\pm 0.00713$ & $\pm 0.00180$ & $\pm 0.00281$ \\
DeepSeek-V3      & $0.27032$ & $0.26620$ & $0.25417$ & $0.24207$ \\
                 & $\pm 0.00225$ & $\pm 0.00301$ & $\pm 0.00196$ & $\pm 0.00326$ \\
Hash (control)   & $0.44665$ & $0.44771$ & $0.44707$ & $0.44652$ \\
                 & $\pm 0.00087$ & $\pm 0.00084$ & $\pm 0.00162$ & $\pm 0.00108$ \\
\bottomrule
\end{tabular}
\end{table}

\Cref{tab:covtype} reports the node-map sweep. \branch{} with $h = 8$ is the
best arm overall ($0.23720$ cross-entropy, $90.21\%$ accuracy, best on both),
but Skywork and Switch tie for second at the same $h$ ($0.24062$ each), and
the margin of $-0.00342$ over them is not clean. Against Switch a paired
$t$-test resolves it only marginally ($t_4 = -2.86$ against a critical value of
$2.776$) and the unpaired Welch test does not ($t = -2.49$); against Skywork
neither test resolves it (paired $t_4 = -1.64$, Welch $t = -2.01$). The
margin over DeepSeek-V3 at $h = 8$ is cleaner ($-0.00487$, paired
$t_4 = -3.25$). With linear node maps \branch{} again has the lowest mean: its
gap to Switch is larger and resolved ($-0.00699$, paired $t_4 = -3.25$), but
its gap to Skywork, the runner-up, is not ($-0.00411$, paired $t_4 = -1.26$).
At $h = 2$ and $h = 4$ Skywork has the lowest mean, by margins of $0.0004$ and
$0.0001$ that are far inside seed noise.

The four Hash arms are four runs of one identical configuration, so the spread
of their seed means isolates pure run-to-run nondeterminism:
$\mathrm{sd} = 0.00053$. The \branch{}-vs-Switch gap at $h = 8$ is well above
this floor yet not resolved by the unpaired test, which is precisely why
pairing is required to see it at all; the equal gap to Skywork is not resolved
even with pairing.

On balance \branch{} attains $\mathrm{CV} = 0.093$ with a $1.174\times$ peak
load at linear ($0.112$ / $1.240\times$ at $h = 8$), against $0.121$ /
$1.210\times$ for Switch and $0.248$ / $1.366\times$ for Skywork, the least
balanced learned router at linear ($0.137$ / $1.247\times$ at $h = 8$); DeepSeek-V3's
controller is tighter still ($0.001$ / $1.002\times$), as expected from a
mechanism that optimizes balance directly. The only dead experts (below $1\%$
of uniform load) are Skywork's at $h = 2$: $6$ of its $80$ expert slots across
the five seeds.

\subsection{HIGGS and YearPredictionMSD}
\label{sec:se-wins}

\begin{table}[t]
\centering
\small
\caption{HIGGS (logloss) and YearPredictionMSD (MSE), $E = 16$, top-$4$,
$5$ paired seeds (mean $\pm$ std). Each router is shown at its best node map,
given in parentheses; for \branch{} we show both $h = 8$ (best loss) and linear
(best locality), since the two differ by $0.0003$ on HIGGS. Tree and index
distances are normalized against the $0.8167$ / $0.3778$ random baselines.
Hash ignores the node map, so its four arms are one configuration run four
times; the best is shown.}
\label{tab:higgs-msd}
\begin{tabular}{lcccccc}
\toprule
 & \multicolumn{3}{c}{HIGGS} & \multicolumn{3}{c}{YearPredictionMSD} \\
\cmidrule(lr){2-4}\cmidrule(lr){5-7}
Method & Logloss $\downarrow$ & Tree $\downarrow$ & Index $\downarrow$ & MSE $\downarrow$ & Tree $\downarrow$ & Index $\downarrow$ \\
\midrule
Skywork          & $\mathbf{0.5247}$ \tiny{(lin)} & $0.812$ & $0.370$ & $0.6757$ \tiny{($h{=}2$)} & $0.813$ & $0.365$ \\
                 & $\pm 0.0005$ & & & $\pm 0.0015$ & & \\
\branch{} ($h{=}8$) & $0.5256$ & $0.736$ & $0.322$ & $0.6767$ & $0.749$ & $0.324$ \\
                 & $\pm 0.0014$ & & & $\pm 0.0027$ & & \\
\branch{} (linear)  & $0.5259$ & $\mathbf{0.710}$ & $\mathbf{0.292}$ & $0.6843$ & $\mathbf{0.718}$ & $\mathbf{0.301}$ \\
                 & $\pm 0.0008$ & & & $\pm 0.0035$ & & \\
DeepSeek-V3      & $0.5264$ \tiny{(lin)} & $0.816$ & $0.380$ & $\mathbf{0.6746}$ \tiny{($h{=}8$)} & $0.813$ & $0.371$ \\
                 & $\pm 0.0013$ & & & $\pm 0.0043$ & & \\
Switch softmax   & $0.5276$ \tiny{($h{=}4$)} & $0.814$ & $0.374$ & $0.6766$ \tiny{($h{=}8$)} & $0.800$ & $0.365$ \\
                 & $\pm 0.0003$ & & & $\pm 0.0038$ & & \\
Hash             & $0.5574$ \tiny{($h{=}4$)} & $0.817$ & $0.378$ & $0.6899$ \tiny{($h{=}4$)} & $0.817$ & $0.378$ \\
                 & $\pm 0.0007$ & & & $\pm 0.0024$ & & \\
\bottomrule
\end{tabular}
\end{table}

\Cref{tab:higgs-msd} reports the two remaining benchmarks, and the outcome is
different. On HIGGS \branch{} ranks second behind Skywork
($+0.0009$, within the combined seed spread); on MSD it ranks fourth, behind
DeepSeek-V3, Skywork, and Switch, though the four are separated by only
$0.0021$ MSE --- $0.014$ years of RMSE --- against seed standard deviations of
$0.0015$--$0.0043$. The honest summary is that on these two datasets all
learned routers are statistically tied, and \branch{} is not the best point
estimate.

\subsection{Routing Locality Across the Suite}
\label{sec:locality-results}
The locality prediction of \cref{sec:communication-locality} is the claim that survives
most cleanly. On all three UCI benchmarks \branch{} is the \emph{only} router
whose co-routing statistics depart materially from the random baseline
(\cref{tab:locality}).

\begin{table}[t]
\centering
\small
\caption{Normalized co-routing distances ($E = 16$, top-$4$; mean over $5$
seeds). Flat routers: range over Switch,
DeepSeek-V3, Skywork, and Hash, each at the arm shown in \cref{tab:main}.
Bold: most local.}
\label{tab:locality}
\begin{tabular}{@{}lcccccc@{}}
\toprule
 & \multicolumn{2}{c}{\branch{} (linear)} & \multicolumn{2}{c}{\branch{} ($h{=}8$)} & \multicolumn{2}{c}{Flat routers} \\
\cmidrule(lr){2-3}\cmidrule(lr){4-5}\cmidrule(lr){6-7}
Dataset & Tree & Index & Tree & Index & Tree & Index \\
\midrule
Forest Covertype      & $\mathbf{0.718}$ & $\mathbf{0.312}$ & $0.742$ & $0.334$ & $0.813$--$0.819$ & $0.371$--$0.385$ \\
HIGGS                 & $\mathbf{0.710}$ & $\mathbf{0.292}$ & $0.736$ & $0.322$ & $0.812$--$0.817$ & $0.370$--$0.380$ \\
YearPredictionMSD     & $\mathbf{0.718}$ & $\mathbf{0.301}$ & $0.749$ & $0.324$ & $0.800$--$0.817$ & $0.365$--$0.378$ \\
Criteo                & $0.813$          & $0.384$          & $0.804$ & $0.370$ & $0.817$--$0.821$ & $0.378$--$0.382$ \\
\midrule
Random baseline       & $0.817$ & $0.378$ & $0.817$ & $0.378$ & $0.817$ & $0.378$ \\
\bottomrule
\end{tabular}
\end{table}

On the three UCI datasets the reduction is $12$--$13\%$ below baseline on tree
distance for linear \branch{} ($8$--$10\%$ at $h = 8$) and $18$--$23\%$ on
index distance ($12$--$15\%$ at $h = 8$), while every flat router sits within
$3.5\%$ of baseline on both --- the negative control behaving exactly as
expected, since for a flat router the expert index is an arbitrary label.
Criteo is the exception: there \branch{} stays at the random
baseline, and with MLP node maps even the flat Skywork router falls below it,
an artifact discussed in \cref{sec:criteo}.

\if false
\subsection{The Node Map Can Matter More Than the Routing Rule}
\label{sec:nodemap-results}
On Forest Covertype the largest effect is not the choice of routing algorithm:
replacing the linear node map with an $h = 8$ MLP improves
every learned router by $0.017$--$0.028$ cross-entropy --- $3.7$--$6.1$ times
the unpaired noise floor of \cref{sec:covtype}, with paired $t$ from $-7.99$ to $-21.44$ --- while the
largest gap \emph{between} routers is $0.0049$ at $h = 8$ ($0.016$ with linear
maps). The ordering
linear $<$ $h{=}2$ $<$ $h{=}4$ $<$ $h{=}8$ holds without exception, so $h = 8$
is not yet saturating. The Hash control moves by $-0.00013$ ($t = -0.19$) and
its load and locality statistics are bit-identical across arms, confirming the
gain is bought by the routing decision and not by the added parameters
elsewhere.

The sign is nonetheless dataset-dependent. MSD agrees with Forest Covertype
(MLP wins for all four learned routers; $-0.0076$ for \branch{}), but HIGGS
reverses it: linear is the best map for Skywork and DeepSeek-V3, and for
\branch{} the linear and $h = 8$ arms are within $0.0003$. HIGGS has $28$
features against $90$ for MSD and $54$ for Covertype, which is consistent with
--- but does not establish --- a story in which the MLP earns its parameters
only when the input is wide enough to compress.

Criteo shows a third pattern. For DeepSeek-V3, Skywork, and \branch{},
replacing the linear node map with an MLP moves held-out logloss by at most
$4.1 \times 10^{-4}$, and none of these differences is significant
($p > 0.1$); \branch{}'s best map, $h = 4$, is $1.9 \times 10^{-4}$ below
linear ($p = 0.19$). Criteo's $256$-dimensional MoE input is the widest in the
suite, so this null result does not fit the width story above. Switch, in
contrast, fails with every MLP node map: all fifteen runs drop
$67.5$--$68.5\%$ of expert dispatches at capacity during training --- close to
the $68.75\%$ expected if every token in a batch selected the same four
experts --- and lose $0.0035$--$0.0046$ logloss relative to the linear map. We
have not diagnosed this failure and treat these arms as a failed configuration
rather than as a routing comparison.

\paragraph{The MLP consistently loosens what makes \branch{} distinctive.}
Across all three UCI datasets, moving \branch{} from a linear node map to
$h = 8$ shifts both locality metrics and load $\mathrm{CV}$ in the same
direction: normalized tree distance moves from $0.718/0.710/0.718$ (remarkably
tight across three unrelated tasks) to $0.742/0.736/0.749$ --- an almost
constant $+0.024$--$0.031$ degradation --- while index distance moves from
$0.312/0.292/0.301$ to $0.334/0.322/0.324$ and load $\mathrm{CV}$ widens from
$0.093/0.227/0.177$ to $0.112/0.231/0.219$ (Covertype / HIGGS / MSD). Within
Covertype alone the per-metric shifts do not clear $95\%$ at five seeds (paired
$t = +1.20, +2.03, +2.36$ against $2.776$), so no single dataset resolves the
effect, but the sign agrees across all nine (dataset $\times$ metric) pairs.
$h = 8$ remains far below the $0.817$ random baseline on every UCI dataset, so
locality is weakened rather than lost --- a regression the theory permits,
since its balance bound covers only linear maps.
On Criteo the balance shift is larger and resolved: \branch{}'s load
$\mathrm{CV}$ rises from $0.115$ with the linear map to $0.21$--$0.38$ with
MLP maps (Welch $p \le 0.01$ at $h = 4$ and $h = 8$), and the share of
dispatches it drops at capacity rises from $0.08\%$ to $2$--$7\%$; locality
cannot loosen there, because the linear map shows none (\cref{sec:criteo}).
Anchoring forces the arrival-weighted \emph{mean} logit at each node to zero,
and that holds for any node map. Balance, however, is a statement about the
\emph{median}. For a linear map the two are bridged by Grünbaum's inequality:
the decision region is a halfspace through the arrival-weighted centroid, so on
a log-concave incoming law each side of the boundary receives at least
$1/\mathrm{e}$ of the arriving traffic, and each child at least
$1/(2\mathrm{e})$ of the parent's soft routing mass
(\cref{cor:linear-soft-anticollapse}). For an MLP the region $\{\mathbf{x} : f_n(\mathbf{x}) \ge \nu_n\}$
is a curved, possibly disconnected superlevel set and no such bound applies: a
node can hold its mean logit at exactly zero while sending most of its traffic
one way. A balance and locality regression is therefore \emph{permitted} by the
MLP in a way it is not by the linear map.

\section{Extended Limitations}
\label{sec:app-limitations}

\paragraph{Scope of the theoretical guarantees.}
Our analysis is conditional. The anti-collapse and expert-optimization
guarantees require log-concave node-conditional distributions and sufficiently
accurate anchor tracking, and the optimization result additionally assumes a
frozen router with strongly convex conditional expert risks; they therefore do
not establish end-to-end convergence of joint router--expert training. The
locality guarantees likewise hold for a fixed placement and are statements
about dispatch patterns rather than about end-to-end throughput.

Two of these conditions are visibly load-bearing rather than technical. First,
the Grünbaum argument that converts a zero mean logit into a bounded mass split
applies only to halfspace decision regions, so it lapses entirely for the MLP
node maps of \cref{sec:nodemap} --- which are also the maps that perform best
on two of the three UCI datasets and give \branch{} its lowest mean logloss on
Criteo. The theory therefore
does not cover the configuration we report as strongest. Second, log-concavity
is assumed of the distribution arriving at each node, a quantity that is itself
shaped by training, and we do not verify it empirically.

\paragraph{The headline task-quality margins are not resolved.}
\branch{} is the best point estimate on two of four benchmarks, but neither win
is clean at five seeds. On Criteo, \branch{} with an $h = 4$ node map leads
Skywork by $7.5\times10^{-5}$ (Welch $p = 0.63$), about a third of a seed
standard deviation, and with the linear node map shared by all routers it
ranks third of the four learned routers, which span only $1.9\times10^{-4}$. On
Forest Covertype the margin over Switch at $h = 8$ is resolved by a paired
$t$-test ($t_4 = -2.86$) but not by the unpaired Welch test, and the critical
value is $2.776$: the conclusion depends on the choice of test. The same
margin over Skywork is resolved by neither test (paired $t_4 = -1.64$, Welch
$t = -2.01$). On HIGGS and
YearPredictionMSD \branch{} is second and fourth respectively, within a
statistical tie of the leaders. The defensible reading of the suite is parity
with the strongest baselines, and we have written
\cref{sec:results} and \cref{sec:criteo,sec:covtype,sec:se-wins} accordingly; a superiority claim on
task quality would require substantially more seeds.

\paragraph{The locality result does not replicate on Criteo.}
On all three UCI benchmarks \branch{} is the only router whose co-routing
statistics fall below the random baseline, by $8$--$13\%$ on tree distance.
On Criteo it shows no such locality: with the linear node map its tree
distance ($0.813$) is indistinguishable from random ($0.817$, $p = 0.48$),
although an untrained \branch{} scores $0.725$. Criteo differs from the UCI
benchmarks in input representation, batch size, and the number of passes over
the data, and we did not vary these independently. We therefore cannot
presently attribute the discrepancy, and until it is attributed the locality
claim should be read as holding on the three UCI benchmarks rather than
universally. The flat-router negative control is also imperfect on Criteo:
with MLP node maps Skywork reaches sub-random tree distances
($0.75$--$0.77$), which we have not diagnosed and provisionally attribute to
an index-order artifact in expert selection rather than to locality.

\paragraph{The MLP node map loosens the properties that motivate the method.}
Across all three UCI datasets, moving \branch{} from a linear node map to
$h = 8$ shifts load $\mathrm{CV}$ ($0.093/0.227/0.177 \to 0.112/0.231/0.219$),
normalized tree distance ($0.718/0.710/0.718 \to 0.742/0.736/0.749$), and index
distance ($0.312/0.292/0.301 \to 0.334/0.322/0.324$) in the same direction.
Within Covertype alone the per-metric shifts do not clear $95\%$ at five seeds
(paired $t = +1.20, +2.03, +2.36$ against $2.776$), so no single dataset
resolves the effect, but the sign agrees in all nine (dataset $\times$ metric)
cells. On Criteo the balance shift is resolved (load $\mathrm{CV}$
$0.115 \to 0.268$ at $h = 4$, Welch $p = 0.005$). That is a direction the
theory permits once Grünbaum's inequality
no longer applies, and it means the configuration with the lowest validation
loss is also the one where the mechanism's guarantees are weakest.

\paragraph{Effect sizes and statistical power.}
On Criteo the dynamic range across the five routing mechanisms is $0.23\%$
relative logloss with the linear node map ($0.25\%$ with each router at its
best node map), almost all of it the gap to the hash control; the four learned
routers span only $0.04\%$. Differences of this magnitude are consequential in
production click-through-rate systems, but they leave little headroom for
distinguishing methods: at five seeds per arm, logloss differences below
${\approx}3.7 \times 10^{-4}$ ($0.08\%$) are not resolvable, and because TPU
execution is not bitwise reproducible, pairing by seed cannot recover power
there. Only the gap to the hash control is unambiguous (Welch $p \le 0.001$
against each learned router with the linear node map). On Forest Covertype
the four hash arms --- four runs of one identical configuration, since
\textsc{HashRouter} ignores the node map --- isolate run-to-run nondeterminism
at $\mathrm{sd} = 0.00053$; the \branch{}-over-Switch margin is well above this
floor yet unresolved without pairing, which is why pairing is required to see
it at all.

\paragraph{Scope of the empirical study.}
We evaluate four tabular benchmarks with a single mixture-of-experts block at
$E = 16$ and top-$4$ routing. The sweep crosses routing algorithm, node map,
and random seed at a single expert count and capacity factor, and no
per-method hyperparameter search was performed --- each baseline runs at
its published defaults. The ranking should therefore be read as holding at
these operating points rather than uniformly. In particular, $h = 8$ is the
largest node-map width tested and on Forest Covertype the improvement is still
monotone in $h$, so the node-map factor is not saturated there; and neither
top-$K$ nor the expert count was varied. All experiments are tabular and single-block: we make no claim about
language models, multi-block architectures where routing decisions couple
across depth, or the substantially larger expert counts at which production MoE
systems operate.

\paragraph{No wall-clock or FLOP claims.}
\branch{} evaluates all $E - 1$ internal node maps for every token, against $E$
expert scores for a flat gate, so the tree shortens the sequential depth of
routing to $\log_2 E$ levels without materially reducing routing work. Our
implementation also dispatches by a per-expert gather/scatter loop that is
unrolled over $E$ in the computation graph, so graph size grows linearly with
the expert count. We therefore make no wall-clock or FLOP claims. The
communication-locality bounds are likewise
statements about dispatch patterns under an assumed placement, not measured
interconnect traffic.

\paragraph{Future work.}
Three items follow directly from the limitations above. First, re-run the
UCI node-map comparison with enough seeds to settle whether the MLP map
degrades balance and locality there (on Criteo its effect on balance is
already resolved), since that determines whether the method's best
configuration is also the one that forfeits its guarantees. Second, attribute
why the locality seen on the UCI benchmarks does not appear on Criteo, for
instance by varying batch size and training length there independently.
Third, extend the analysis to a deep multilayer residual
formulation in which routing trees compose across depth --- jointly evolving
representations, layerwise anchors, coupling between routes at successive
layers, and end-to-end optimization in that architecture all remain open.
Evaluating at a scale where conditional computation is known to pay for itself
would additionally place the routing comparison in a regime where the model
class is not itself the limiting factor.
\fi

\section{Description of Datasets}
\label{sec:datasets}

We evaluate the considered methods on four widely used tabular benchmarks spanning binary classification, multiclass classification, and regression. These datasets vary substantially in size, feature type, and prediction task.

\paragraph{Criteo.}
The Criteo 1TB Click Logs dataset is a large-scale benchmark for click-through-rate (CTR) prediction. Each instance corresponds to a display advertisement served by Criteo, and the binary label indicates whether the ad was clicked. The training data consist of 24 chronologically ordered daily files and contain 13 numerical, predominantly count-based features and 26 anonymized categorical features. The categorical values are hashed into 32-bit integers, and some entries may be missing. We use this dataset as a large-scale binary classification task \citep{criteo_1tb_click_logs}.

\paragraph{Forest Covertype.}
The Forest Covertype dataset consists of $581{,}012$ observations of $30 \times 30$ meter forest patches from the Roosevelt National Forest in Colorado. The task is to predict the dominant forest cover type among seven classes from 54 cartographic variables, including elevation, slope, aspect, hillshade measurements, wilderness-area indicators, and soil-type indicators. It therefore provides a multiclass classification benchmark with a mix of continuous and binary features \citep{blackard1999covertype,blackard1998covertype_uci}.

\paragraph{HIGGS.}
The HIGGS dataset is a binary classification dataset derived from simulated high-energy particle-collision events. Its goal is to distinguish events corresponding to a signal process that produces Higgs bosons from background events. The dataset contains $11$ million instances, each represented by 28 real-valued features: 21 low-level kinematic variables measured by particle detectors and seven high-level features derived from the low-level variables. This dataset is a standard large-scale benchmark for binary classification \citep{baldi2014searching,whiteson2014higgs_uci}.

\paragraph{YearPredictionMSD.}
The YearPredictionMSD dataset is a regression benchmark constructed from the Million Song Dataset. The task is to predict a song's release year from 90 audio features that summarize timbre characteristics. It contains $515{,}345$ songs, primarily Western commercial recordings released between 1922 and 2011. We use it to assess performance on a large-scale continuous prediction problem \citep{bertinmahieux2011million,bertinmahieux2011yearpredictionmsd_uci}.
\end{document}